\documentclass[twoside]{article}

\pdfoutput=1

\usepackage[accepted]{aistats2019}
\usepackage[authoryear,round]{natbib}

\usepackage{amsthm,amsmath,amssymb}
\usepackage{mathtools}
\usepackage{thmtools}

\newtheorem{theorem}{Theorem}[section]

\newtheorem{definition}[theorem]{Definition}

\newtheorem{remark}[theorem]{Remark}

\usepackage[hidelinks]{hyperref}
\usepackage{graphicx}
\usepackage{xcolor}
\usepackage{mathrsfs}
\usepackage{enumitem}
\usepackage{verbatim}
\usepackage{subfig,calc}

\newcommand{\calcd}{\mathrm{d}}
\newcommand{\wass}{\mathrm{W}}
\newcommand{\SW}{\mathrm{SW}}
\newcommand{\PW}{\mathrm{PW}}
\newcommand{\dimension}{d}

\newcommand{\distone}{\eta}
\newcommand{\disttwo}{\mu}
\newcommand{\distthree}{\zeta}
\newcommand{\projNumber}{N}
\newcommand{\projIx}{n}
\newcommand{\projIxAlt}{n^\prime}
\newcommand{\suppNumber}{M}
\newcommand{\suppIx}{m}
\newcommand{\distOnePoint}{\mathbf{x}}
\newcommand{\distTwoPoint}{\mathbf{y}}
\newcommand{\distThreePoint}{\mathbf{z}}
\newcommand{\distMatrix}{\mathbf{D}}
\newcommand{\birkhoff}{\mathscr{B}}
\newcommand{\ortgroup}[1]{\mathscr{O}(#1)}
\newcommand{\symgroup}[1]{\mathcal{S}_{#1}}

\newcommand{\targetVal}{\theta}
\newcommand{\estim}[1]{\hat{\targetVal}_{#1}}

\newcommand{\E}{\mathbb{E}}
\newcommand{\R}[1]{\mathbb{R}^{#1}}
\newcommand{\Prob}{\mathbb{P}}
\newcommand{\ortd}{d}

\usepackage{algorithm}
\usepackage{algorithmic}
\usepackage{cleveref} 

\DeclareMathOperator*{\argmin}{arg\,min}

\begin{document}

%

%
\runningauthor{Rowland*, Hron*, Tang*, Choromanski, Sarlos, Weller}

\twocolumn[

\aistatstitle{Orthogonal Estimation of Wasserstein Distances}

\aistatsauthor{Mark Rowland$^{*1}$  Jiri Hron$^{*1}$  Yunhao Tang$^{*2}$  Krzysztof Choromanski$^{3}$  Tamas Sarlos$^{4}$ Adrian Weller$^{1,5}$}

\aistatsaddress{$^{1}$University of Cambridge, $^{2}$Columbia University, $^{3}$Google Brain, $^{4}$Google Research, $^{5}$The Alan Turing Institute}

]

\begin{abstract}
  Wasserstein distances are increasingly used in a wide variety of applications in machine learning. Sliced Wasserstein distances form an important subclass which may be estimated efficiently through one-dimensional sorting operations. In this paper, we propose a new variant of sliced Wasserstein distance, study the use of orthogonal coupling in Monte Carlo estimation of Wasserstein distances and draw connections with stratified sampling, and evaluate our approaches experimentally in a range of large-scale experiments in generative modelling and reinforcement learning.
\end{abstract}

\section{INTRODUCTION}


Wasserstein distances are a method of measuring distance between probability distributions, and are widely used in image processing \citep{RabinEtAl,Bonneel}, probability \citep{GradientFlows}, physics \citep{JKO} and economics \citep{OTEconomics}, and are increasingly used in machine learning \citep{wgan,improved-wgan,COT}. 
Wasserstein distances are popular as they take into account spatial information (unlike total variation distance), and can be defined between continuous and discrete distributions (unlike the Kullback-Leibler divergence). These factors have made Wasserstein distances particularly popular in defining objectives for generative modelling \citep{wgan,improved-wgan}.

However, exact computation of Wasserstein distances is costly, as it requires the solution of an optimal transport problem. This has motivated the development of a wide variety of computationally tractable approximations (see for example \citep{Sinkhorn,StocOptOT}). 
The sliced Wasserstein distance was proposed by \citet{RabinEtAl}, and exploits the fact that one-dimensional instances of optimal transport problems can be solved much quicker than the general case, working directly with one-dimensional projections of the probability distributions in question; see Section \ref{sec:wasserstein} for a more detailed exposition.

Although the move from Wasserstein distance to sliced Wasserstein distance often yields significant gains in computational tractability, we highlight two issues that remain. Firstly, the focus of sliced Wasserstein distance on one-dimensional marginals of probability distributions can lead to poorer quality results than true Wasserstein distance \citep{Bonneel}. Secondly, the evaluation of sliced Wasserstein distance itself generally requires Monte Carlo estimation, and thus enough Monte Carlo samples must be used to control the random fluctuations introduced by this stochastic method.

In this paper, we investigate these two shortcomings of the sliced Wasserstein distance and propose a new distance which incorporates the computational benefits of the sliced Wasserstein distance, whilst still retaining information from the high-dimensional distances. Given recent strong results based on Wasserstein distances and their tractable approximations, exploring related methods is an important area of study. However, our initial empirical results demonstrate only small benefits in some contexts.

We study a Monte Carlo sampling scheme for more efficient estimation of both sliced Wasserstein distance, and our novel variant. We conclude this section by highlighting our main contributions:
\begin{itemize}[leftmargin=*]
	\item In Section \ref{sec:projwasserstein}, we introduce a new variant of Wasserstein distance, which we term \emph{projected Wasserstein distance}, which incorporates aspects of both sliced Wasserstein distance and true Wasserstein distance.
	\item In Section \ref{sec:ortestimation}, we study a Monte Carlo method for more accurate estimation of sliced Wasserstein and projected Wasserstein distances, based on orthogonal couplings of projection vectors, and provide theoretical analysis and connections to notions of stratified sampling.
	\item In Section \ref{sec:experiments}, we empirically evaluate the performance of projected Wasserstein distance, and orthogonally-coupled estimation, on a variety of tasks, including high-dimensional generative modelling and reinforcement learning.
\end{itemize}

\section{WASSERSTEIN AND SLICED WASSERSTEIN DISTANCES}\label{sec:wasserstein}


We briefly review Wasserstein and sliced Wasserstein distances, and their role in modern machine learning. For a more detailed review of theory of Wasserstein distances and optimal transport more generally, see \citet{Villani}. In this paper, we consider these distances in the specific case of Euclidean base space $(\mathbb{R}^d, \| \cdot \|_2)$, and first recall the definition of Wasserstein distance at this level of generality.

\begin{definition}[Wasserstein distances]
	Let $p \geq 1$. The $p$-Wasserstein distance is defined on $\mathscr{P}_p(\mathbb{R}^\dimension)$, the set of probability distributions over $\mathbb{R}^\dimension$ with finite $p$\textsuperscript{th} moment, by:
	\begin{align*}
		\wass_p(\distone, \disttwo) =
			\left\lbrack
				\inf_{\gamma \in \Gamma(\distone,\disttwo)}
					\int_{\mathbb{R}^\dimension \times \mathbb{R}^\dimension}
						\| \mathbf{x}- \mathbf{y} \|^p_2
						\gamma(\calcd \mathbf{x}, \calcd \mathbf{y})
			\right\rbrack^{1/p} \, ,
	\end{align*}
	for all $\distone, \disttwo \in \mathscr{P}_p(\mathbb{R}^\dimension)$, 
	where $\Gamma(\distone, \disttwo) \subseteq \mathscr{P}(\mathbb{R}^\dimension \times \mathbb{R}^\dimension)$ is the set of joint distributions over $\mathbb{R}^\dimension \times \mathbb{R}^\dimension$ for which the marginal on the first $\dimension$ coordinates is $\distone$, and the marginal on the last $\dimension$ coordinates is $\disttwo$. An element $\gamma\in \Gamma(\distone, \disttwo)$ achieving the infimum is called an optimal transport plan between $\distone$ and $\disttwo$.
\end{definition}

A common problem in the image processing literature using Wasserstein distance is that of \emph{computing a barycentre} for a collection of measures \citep{fastwassersteinbarycentre,parallelWB}. In this problem, a collection of measures $\{\mu_j\}_{j=1}^J \subseteq \mathscr{P}_p(\mathbb{R}^d)$ and a collection of weightings $(\lambda_{j})_{j=1}^J \in \mathbb{R}_{>0}^J$ are given, and the following objective is posed:
\begin{align}\label{eq:barycentre}
	\argmin_{\eta \in \mathcal{Q}} \sum_{j=1}^J \lambda_j \wass^p_p(\eta, \mu_j)  \, .
\end{align}
In some sense, this generalises the notion of finding a centre of mass of a collection of points in Euclidean space to the ``centre of mass'' of a collection of probability measures. A special case of the barycentre problem which is common in the machine learning literature is that of \emph{distribution learning}, for which $J=1$, reducing Problem \eqref{eq:barycentre} to the following optimisation problem
\begin{align}\label{eq:distlearn}
\argmin_{\eta \in \mathcal{Q}} \wass^p_p(\eta, \mu) \, ,
\end{align}
for a given $\mu \in \mathscr{P}_p(\mathbb{R}^d)$, and a space of probability distributions $\mathcal{Q} \subseteq \mathscr{P}_p(\mathbb{R}^d)$. A typical application is in deep generative modelling, in which $\mu$ is the empirical distribution corresponding to some dataset, and $\mathcal{Q}$ is a set of distributions parametrised by a neural network \citep{wgan}.

One of the reasons that Wasserstein distances are not more commonly used in machine learning is that their evaluation requires solving an optimisation problem. As an example, computing $\wass^p_p(\distone, \disttwo)$ in the particular case where
\begin{align}\label{eq:finite_dists}
	\distone = \frac{1}{\suppNumber} \sum_{\suppIx=1}^\suppNumber \delta_{\distOnePoint_\suppIx} \, , \
	\disttwo = \frac{1}{\suppNumber} \sum_{\suppIx=1}^\suppNumber \delta_{\distTwoPoint_\suppIx}  \, ,
\end{align}
is expressed as the following linear program:
\begin{align}\label{eq:linprog}
\max_{\mathbf{T} \in \birkhoff_\suppNumber}\ \langle \mathbf{T}, \distMatrix \rangle / \suppNumber \, .
\end{align}
Here, Dirac delta function $\delta_{\distOnePoint_\suppIx}$ denotes the probability density function of idealized point mass at sample data point $\distOnePoint_\suppIx$,
$\distMatrix \in \mathbb{R}^{\suppNumber \times \suppNumber}$ specifies costs between elements of the supports of $\distone$ and $\disttwo$, so that $\distMatrix_{ij} = \| \distOnePoint_i - \distTwoPoint_j \|_2^p$ for all $i,j \in [\suppNumber]$, and $\birkhoff_\suppNumber = \{ \mathbf{A} \in \mathbb{R}^{\suppNumber \times \suppNumber} | \mathbf{A} \mathbf{1} = \mathbf{1} \, , \ \mathbf{A}^\top \mathbf{1} = \mathbf{1} \}$ is the Birkhoff polytope. For linear programs of this specific form, the most efficient known solution methods are matching algorithms, which can achieve $\mathcal{O}(\suppNumber^{5/2} \log \suppNumber)$ complexity; for large $M$, this cost can quickly become infeasible as an inner loop in a learning algorithm. Therefore \citep{wgan} \emph{heuristically} optimized equivalent dual form of $\wass^p_p$.

In contrast, for \emph{one-dimensional} probability distributions $\eta, \mu \in \mathscr{P}(\mathbb{R})$, the optimal transport plan $\gamma^\star \in \Gamma(\eta, \mu)$ may be written down analytically, and in the case of finitely-supported distributions, the computation of Wasserstein distance amounts to sorting the support of the distributions in question on the real line, with a complexity of only $\mathcal{O}(\suppNumber \log \suppNumber)$. This motivates an alternative to the Wasserstein distance which exploits the computational ease of optimal transport on the real line: \emph{sliced Wasserstein distances} \citep{RabinEtAl}. For a vector $\mathbf{v} \in S^{d-1}$, where $S^{d-1}$ is the~unit sphere in $\mathbb{R}^d$, we define the projection map $\Pi_\mathbf{v} : \mathbb{R}^d \rightarrow \mathbb{R}$ by $\Pi_\mathbf{v}(\mathbf{x}) = \langle \mathbf{v}, \mathbf{x} \rangle$, for all $\mathbf{x} \in \mathbb{R}^d$, and denote projection of distribution $\distone$ by $(\Pi_\mathbf{v})_\# \distone$. With this notion established, we may now precisely define sliced Wasserstein distances.

\begin{definition}[Sliced Wasserstein distances]
	Let $p \geq 1$. The $p$-sliced Wasserstein distance is defined on $\mathscr{P}_p(\mathbb{R}^\dimension)$ by
	\begin{align}\label{eq:sw}
		& \SW_p^p(\distone, \disttwo)  
		=
				\mathbb{E}_{\mathbf{v} \sim \mathrm{Unif}(S^{d-1})}\left\lbrack
				\wass^p_p( (\Pi_\mathbf{v})_\# \distone, (\Pi_\mathbf{v})_\# \disttwo)
				\right\rbrack
		\, , 
	\end{align}
	for all $\distone, \disttwo \in \mathscr{P}_p(\mathbb{R}^\dimension)$.
\end{definition}

Due to the intractable expectation appearing in Equation \eqref{eq:sw}, the sliced Wasserstein distance is typically estimated via Monte Carlo sampling
\begin{align}\label{eq:sw-estimator}
\widehat{\SW}_p^p(\distone, \disttwo)
=
\frac{1}{N} \sum_{n=1}^{N}
\wass^p_p((\Pi_{\mathbf{v}_n})_\# \distone, (\Pi_{\mathbf{v}_n})_\# \disttwo)\,  , 
\end{align}
where $\mathbf{v}_1,\ldots,\mathbf{v}_N \overset{\mathrm{i.i.d.}}{\sim} \mathrm{Unif}(S^{d-1})$. The precise algorithmic steps for this estimation in the case of empirical distributions $\eta = \frac{1}{\suppNumber} \sum_{\suppIx=1}^\suppNumber \delta_{\distOnePoint_\suppIx}$, $\mu = \frac{1}{\suppNumber} \sum_{\suppIx=1}^\suppNumber \delta_{\distTwoPoint_\suppIx}$ are given in Algorithm \ref{alg:sw}.

\begin{algorithm}
	\caption{Sliced Wasserstein estimation}
	\label{alg:sw}
	\begin{algorithmic}[1]
		\REQUIRE $\eta = \frac{1}{\suppNumber} \sum_{\suppIx=1}^\suppNumber \delta_{\distOnePoint_\suppIx}$, \ $\mu = \frac{1}{\suppNumber} \sum_{\suppIx=1}^\suppNumber \delta_{\distTwoPoint_\suppIx}$
		\FOR{ $\projIx=1$ \TO $\projNumber$}
			\STATE {Sample $\mathbf{v}_\projIx \sim \mathrm{Unif}(S^{d-1})$.}
			\STATE{Compute projected distributions:}
			\STATE{
				\ \ $\textstyle (\Pi_{\mathbf{v}_\projIx})_\# \distone = \frac{1}{\suppNumber} \sum_{\suppIx=1}^\suppNumber \delta_{\langle \mathbf{v}_\projIx, \distOnePoint_\suppIx \rangle}$}
			\STATE{\ $\textstyle (\Pi_{\mathbf{v}_\projIx})_\# \disttwo = \frac{1}{\suppNumber} \sum_{\suppIx=1}^\suppNumber \delta_{\langle \mathbf{v}_\projIx, \distTwoPoint_\suppIx \rangle}$}
			\STATE{Compute sorted lists of supports:}
			\STATE{\ $\textstyle (w_\suppIx)_{\suppIx=1}^\suppNumber \leftarrow \texttt{sort}((\langle \mathbf{v}_\projIx , \distOnePoint_\suppIx \rangle)_{\suppIx=1}^\suppNumber)$}
			\STATE{\ $\textstyle (z_\suppIx)_{\suppIx=1}^\suppNumber \leftarrow \texttt{sort}((\langle \mathbf{v}_\projIx , \distTwoPoint_\suppIx \rangle)_{\suppIx=1}^\suppNumber)$}
			\STATE{Compute one-dimensional Wasserstein distance:}
			\STATE{\ $\textstyle \wass^p_p( (\Pi_{\mathbf{v}_\projIx})_\# \distone, (\Pi_{\mathbf{v}_\projIx})_\# \disttwo)\! =\! \frac{1}{\suppNumber} \sum_{\suppIx=1}^\suppNumber |w_\suppIx - z_\suppIx|^p $}\label{eqline:sw:distances}
		\ENDFOR
		\RETURN		
		$\frac{1}{\projNumber} \sum_{\projIx=1}^\projNumber \wass^p_p((\Pi_{\mathbf{v}_\projIx})_\# \distone,(\Pi_{\mathbf{v}_\projIx})_\# \disttwo)$
	\end{algorithmic}
\end{algorithm}

Sliced Wasserstein distances were originally proposed for image processing applications \citep{RabinEtAl,Bonneel} to avoid expensive computation with true Wasserstein distances. More recently, sliced Wasserstein distances have also been used in deep generative modelling \citep{swae,swgan,swgenmodels}.

\section{THE PROJECTED WASSERSTEIN DISTANCE}\label{sec:projwasserstein}


Whilst sliced Wasserstein distances bypass the computational bottleneck for Wasserstein distances (namely, solving the linear program in Problem \eqref{eq:linprog}) required for each evaluation, they exhibit different behaviour from true Wasserstein distance, which in many cases may be undesirable. We offer an intuition as to why the qualitative properties of sliced Wasserstein and true Wasserstein distances differ. Inspecting Algorithm \ref{alg:sw}, we note that within the main loop, the random vector $\mathbf{v}_{\projIx}$ plays \emph{two} roles: firstly, as the determiner of the matching between the projected particles $(\langle \mathbf{v}_\projIx, \distOnePoint_\suppIx \rangle)_{\suppIx = 1}^\suppNumber$, $(\langle \mathbf{v}_\projIx, \distTwoPoint_\suppIx \rangle)_{\suppIx = 1}^\suppNumber$; and secondly, in the computation of the distances between the projected particles. Roughly speaking, this may be thought of as introducing some type of ``bias''.

This is similar in flavour to the phenomenon observed by \citet{doubleQlearning} in the context of Q-learning, in which the maximum of a collection of samples is shown to be biased (over)estimate of the maximum of the corresponding population means. Indeed, \citet{doubleQlearning} observes that this phenomenon has a long history (see e.g. \citet{OptimisersCurse,PostdecisionSurprises}).

Suppose we were to use separate projections to compute the distances at Line \ref{eqline:sw:distances} of Algorithm \ref{alg:sw}. More precisely, suppose we sample $\mathbf{v}^\prime_\projIx \sim \mathrm{Unif}(S^{d-1})$ independently of $\mathbf{v}_\projIx$, and introduce the notation $\sigma_{\mathbf{v}_\projIx} : [\suppNumber] \rightarrow [\suppNumber]$ for the bijective mapping with the property that $\langle \mathbf{v}_\projIx, \distOnePoint_i \rangle < \langle \mathbf{v}_\projIx, \distOnePoint_j \rangle \implies \langle \mathbf{v}_\projIx, \distTwoPoint_{\sigma_{\mathbf{v}_\projIx}(i)} \rangle \leq \langle \mathbf{v}_\projIx, \distTwoPoint_{\sigma_{\mathbf{v}_\projIx}(j)} \rangle$. Now consider replacing Line \ref{eqline:sw:distances} of Algorithm \ref{alg:sw} with the following computation:
\begin{align}\label{eq:unbiaseddistance}
\frac{1}{\suppNumber} \sum_{\suppIx=1}^\suppNumber | \langle \mathbf{v}^\prime_\projIx, \distOnePoint_i \rangle -\langle \mathbf{v}^\prime_\projIx, \distTwoPoint_{\sigma_{\mathbf{v}_\projIx}(i)} \rangle |^p \, ,
\end{align}
noting that in degenerate cases there may exist more than one such $\sigma_\mathbf{v}$, in which case we select uniformly from this set of induced couplings. 
The computation in Expression \eqref{eq:unbiaseddistance} removes the source of ``bias'' identified previously.

Further, we observe that in the special case of $p=2$, the use of a second projection to compute the costs can itself be interpreted as an unbiased estimator (up to a multiplicative scaling) of the original pairwise costs themselves. This motivates the \emph{projected Wasserstein distance}, which we define formally below, along with the prerequisite notion of \emph{induced couplings}.

\begin{definition}[Induced couplings]\label{def:couplings}
	Given two empirical distributions $\distone = \frac{1}{\suppNumber}\sum_{\suppIx=1}^{\suppNumber}\delta_{\distOnePoint_\suppIx}$ $\disttwo=\frac{1}{\suppNumber}\sum_{\suppIx=1}^{\suppNumber}\delta_{\distTwoPoint_\suppIx}$ and a vector $\mathbf{v} \in S^{d-1}$, we define the \emph{couplings induced by $\mathbf{v}$} to be the set $\Sigma_\mathbf{v}$ of bijective maps $[\suppNumber] \rightarrow [\suppNumber]$ that specify an optimal matching for the projected particles $(\langle \mathbf{v}, \distOnePoint_\suppIx \rangle)_{\suppIx = 1}^\suppNumber$, $(\langle \mathbf{v}, \distTwoPoint_\suppIx \rangle)_{\suppIx = 1}^\suppNumber$, in the sense that a matching $\sigma :[\suppNumber] \rightarrow [\suppNumber]$ is optimal when the condition $\langle \mathbf{v}, \distOnePoint_i \rangle < \langle \mathbf{v}, \distOnePoint_j \rangle$ iff $\langle \mathbf{v}, \distTwoPoint_{\sigma_{\mathbf{v}}(i)} \rangle < \langle \mathbf{v}, \distTwoPoint_{\sigma_{\mathbf{v}}(j)} \rangle$ is satisfied. Note that typically $\Sigma_\mathbf{v}$ is a set of size one, but in degenerate cases, there may be more than one induced coupling $\sigma_\mathbf{v}$ for a given vector $\mathbf{v}$.
\end{definition}
	
\begin{definition}[Projected Wasserstein distances]
	For $p \geq 1$, the $p$-projected Wasserstein distance between $\distone$ and $\disttwo$ is defined as:
	\begin{align}\label{eq:pw}
		\textstyle
		&\PW_p^p
		  \left(
			\distone, \disttwo
		  \right) 
		=
			\mathbb{E}_{\mathbf{v} \sim \mathrm{Unif}(S^{d-1})} \!\!
			\left\lbrack\frac{1}{\suppNumber} \!\!				
				\sum_{\suppIx=1}^\suppNumber \|\mathbf{x}_\suppIx - \mathbf{y}_{\sigma_{\mathbf{v}}(\suppIx)} \|_2^p	
			\right\rbrack
		,  
	\end{align}
	where $\sigma_\mathbf{v} : [\suppNumber] \rightarrow [\suppNumber]$ is the coupling induced by $\mathbf{v}$.
\end{definition}

Projected Wasserstein distances are thus an alternative to Wasserstein distances that enjoy similar computational efficiency to sliced Wasserstein distances, but correct for the ``bias'' implicit in the definition of sliced Wasserstein distances. In the remainder of this section, we explore a variety of theoretical and computational aspects of projected Wasserstein distances, and in Section \ref{sec:experiments} we explore the use of projected Wasserstein distance as an objective in deep generative modelling and reinforcement learning. 

\subsection{Theoretical properties}

Having motivated the projected Wasserstein distance, we now establish some of its basic properties.

\begin{restatable}{proposition}{propMetric}
	Projected Wasserstein distance $\PW_p$ is a metric on the space $\mathscr{P}_{(\suppNumber)}(\mathbb{R}^d) = \{ \frac{1}{\suppNumber} \sum_{\suppIx=1}^\suppNumber \delta_{\distOnePoint_\suppIx} | \distOnePoint_\suppIx \in \mathbb{R}^\dimension \ \ \text{for all} \ \suppIx \in [\suppNumber] \} \subset \mathscr{P}(\mathbb{R}^d)$.
\end{restatable}

\begin{restatable}{proposition}{propIneq}
	We have the following inequalities
	\begin{align}
		\SW_p(\eta, \mu) \leq
		\wass_p(\eta, \mu) \leq
		\PW_p(\eta, \mu) \, ,
	\end{align}
	for all $\eta, \mu \in \mathscr{P}_{(\suppNumber)}(\mathbb{R}^d)$, for all $p \geq 1$.
\end{restatable}

\subsection{Monte Carlo estimation of PW distance}

Just as sliced Wasserstein distance requires Monte Carlo estimation, so too does projected Wasserstein distance. The estimation algorithm is similar to Algorithm \ref{alg:sw} for sliced Wasserstein distance (as our motivation for projected Wasserstein distance might suggest), and is presented as Algorithm \ref{alg:pw}; the crucial difference is the contribution calculation at Line 9. Within Algorithm \ref{alg:pw}, \texttt{argsort} can be taken to be any subroutine that computes an induced coupling between the projected samples. One implementation that runs in $\mathcal{O}(\suppNumber \log \suppNumber)$ time consists of sorting the two lists of real numbers, keeping track of the permutations that sort them, and then computing one with the inverse of the other.

\begin{algorithm}
	\caption{Projected Wasserstein estimation}
	\label{alg:pw}
	\begin{algorithmic}[1]
		\REQUIRE $\eta = \frac{1}{\suppNumber} \sum_{\suppIx=1}^\suppNumber \delta_{\distOnePoint_\suppIx}$, \ $\mu = \frac{1}{\suppNumber} \sum_{\suppIx=1}^\suppNumber \delta_{\distTwoPoint_\suppIx}$
		\STATE Sample $(\mathbf{v}_\projIx)_{\projIx=1}^\projNumber \overset{\mathrm{i.i.d.}}{\sim} \mathrm{Unif}(S^{d-1})$
		\FOR{ $\projIx=1$ \TO $\projNumber$}
		\STATE {Compute projected distributions:}
		\STATE{\ \ $\textstyle (\Pi_{\mathbf{v}_\projIx})_\# \distone = \frac{1}{\suppNumber} \sum_{\suppIx=1}^\suppNumber \delta_{\langle \mathbf{v}_\projIx, \distOnePoint_\suppIx \rangle}$}
		\STATE{\ \ $\textstyle (\Pi_{\mathbf{v}_\projIx})_\# \disttwo = \frac{1}{\suppNumber} \sum_{\suppIx=1}^\suppNumber \delta_{\langle \mathbf{v}_\projIx, \distTwoPoint_\suppIx \rangle}$}
		\STATE{Compute optimal matching for projected distributions:}
		\STATE{\ \ $\sigma_{\mathbf{v}_\projIx} \leftarrow \texttt{argsort}((\langle \mathbf{v}_\projIx , \distOnePoint_\suppIx \rangle)_{\suppIx=1}^\suppNumber, (\langle \mathbf{v}_\projIx , \distTwoPoint_\suppIx \rangle)_{\suppIx=1}^\suppNumber)$}
		\STATE{Compute contribution from coupling:}
		\STATE{\ \ $\textstyle  \frac{1}{\suppNumber} \sum_{\suppIx=1}^\suppNumber \|\distOnePoint_\suppIx - \distTwoPoint_{\sigma_{\mathbf{v}_\projIx}(\suppIx)}\|_2^p$
		} \label{algline:pwdist}
		\ENDFOR
		\RETURN
		$\frac{1}{\projNumber} \sum_{\projIx=1}^\projNumber \!  \frac{1}{\suppNumber}\! \sum_{\suppIx=1}^\suppNumber \|\distOnePoint_\suppIx - \distTwoPoint_{\sigma_{\mathbf{v}_\projIx}(\suppIx)}\|_2^p $
	\end{algorithmic}
\end{algorithm}

Algorithm \ref{alg:pw} thus allows any method utilising Wasserstein or sliced Wasserstein distances, such as Problems \eqref{eq:barycentre} and \eqref{eq:distlearn} instead to use projected Wasserstein distance.

\subsection{Johnson-Lindenstrauss estimation of PW distance}

To conclude this section, we discuss several variations on Algorithm \eqref{alg:pw} which may allow for more efficient estimation of projected Wasserstein distance. These variants are motivated by the difference in computational burden between sliced Wasserstein and projected Wasserstein distances; whilst Algorithm \ref{alg:sw} for Monte Carlo estimation of sliced Wasserstein distances deals entirely with one-dimensional projections, Algorithm \ref{alg:pw} for estimation of the projected Wasserstein distance requires computation of distances between the original (unprojected) datapoints.

However, a key observation is that in the course of computing induced couplings in Algorithm \ref{alg:pw}, many one-dimensional projections of the support of the distributions concerned have been computed. These projections can be pooled, and thus considered collectively as constituting a \emph{random projection} of the support of the distribution, in the vein of the Johnson-Lindenstrauss transform \citep{JLT}. This random projection can then be used to estimate distances between support points of the distributions, without having to work with the original high-dimensional points themselves. More concretely, this can be achieved by replacing Line \ref{algline:pwdist} of Algorithm \ref{alg:pw} with the following computation:
\begin{align}
	\frac{1}{\suppNumber} \sum_{\suppIx=1}^\suppNumber \frac{d}{\projNumber} \sum_{\projIxAlt=1}^{\projNumber} |\langle \mathbf{v}_{\projIxAlt}, \mathbf{x}_{\suppIx} - \mathbf{y}_{\sigma_{\mathbf{v}_{\projIx}}(\suppNumber)} \rangle|^p \, .
\end{align}
In particular, in the case $p=2$, this yields an unbiased estimator of the distance computed in Line \ref{algline:pwdist} of Algorithm \ref{alg:pw}.

\section{ORTHOGONAL ESTIMATION FOR SLICED AND PROJECTED WASSERSTEIN DISTANCES}\label{sec:ortestimation}


Having introduced the projected Wasserstein distance, we now turn to the second contribution in this paper: developing understanding of improved methods for estimating both sliced Wasserstein and projected Wasserstein distances. \citet{AutomatedColorGrading} argue implicitly for using an orthogonal coupling of projection vectors in estimating the sliced Wasserstein distance, and \cite{swgenmodels} put forward an approach to generative modelling where a set of orthogonal projection directions are \emph{learnt} from data, to be used in the context of sliced Wasserstein estimation.

Here, we consider the general approach of using orthogonal projection directions within sliced Wasserstein and projected Wasserstein distances. We show the (perhaps surprising) results that: (i)~there is a strong connection between orthogonal coupling of projection directions and notions of stratified sampling; (ii)~contrary to the intuition of \citet{AutomatedColorGrading}, using orthogonal projection directions can actually \emph{worsen} the performance of sliced Wasserstein estimation (as measured by estimator variance); but (iii)~orthogonal projection directions \emph{always} lead to an improvement in estimator variance for the projected Wasserstein distance in the case $\suppNumber = 2$; we conjecture that this holds more generally. 
Besides the~motivation presented in \Cref{sec:projwasserstein}, the~projected Wasserstein distances therefore serve an~important role in our theoretical understanding of the impact of orthogonally coupled projection directions on estimation of the~sliced Wasserstein distances.
Details on how to perform practical Monte Carlo sampling from $\mathrm{UnifOrt}(S^{d-1}; N)$, as well as computationally efficient approximate sampling algorithms, are provided in the Appendix.

\subsection{Orthogonal couplings}

To make precise the notion of orthogonal projection directions, we first make a preliminary definition.

\begin{definition}\label{def:unifort}
	Let $N \leq d$. 
	The probability distribution $\mathrm{UnifOrt}(S^{d-1};N) \in \mathscr{P}((S^{d-1})^N)$ is defined as the joint distribution of $N$ rows of a random orthogonal matrix drawn from Haar measure on the orthogonal group $\ortgroup{d}$. If $N$ is a multiple of d, we define the distribution $\mathrm{UnifOrt}(S^{d-1};N)$ to be that given by concatenating $N/d$  independent copies of random variables drawn from $\mathrm{UnifOrt}(S^{d-1};d)$.
\end{definition}

A collection of random vectors $(\mathbf{v}_\projIx)_{\projIx=1}^\projNumber$ drawn from $\mathrm{UnifOrt}(S^{d-1};N)$ has the property that each random vector $\mathbf{v}_\projIx$ is marginally distributed as $\mathrm{Unif}(S^{d-1})$, and all vectors are mutually orthogonal almost surely. The broad idea is to replace the i.i.d.\ projection directions $(\mathbf{v}_\projIx)_{\projIx=1}^\projNumber$ appearing in Algorithms \ref{alg:sw} and \ref{alg:pw} with a sample from $\mathrm{UnifOrt}(S^{d-1}; N)$; Algorithm \ref{alg:sw-ort} specifies this adjustment precisely in the case of sliced Wasserstein estimation, with the new sampling mechanism shown in red. The adjustment for projected Wasserstein estimation is analogous, and is given in the appendix due to space constraints.

\begin{algorithm}
	\caption{Orthogonal sliced Wasserstein estimation}
	\label{alg:sw-ort}
	\begin{algorithmic}[1]
		\REQUIRE $\eta = \frac{1}{\suppNumber} \sum_{\suppIx=1}^\suppNumber \delta_{\distOnePoint_\suppIx}$, \ $\mu = \frac{1}{\suppNumber} \sum_{\suppIx=1}^\suppNumber \delta_{\distTwoPoint_\suppIx}$
		\STATE \textcolor{red}{Sample $(\mathbf{v}_\projIx)_{\projIx=1}^\projNumber \sim \mathrm{UnifOrt}(S^{d-1}; \projNumber)$}
		\FOR{ $\projIx=1$ \TO $\projNumber$}
		\STATE {Compute projected distributions:}
		\STATE{\ \ $\textstyle (\Pi_{\mathbf{v}_\projIx})_\# \distone = \frac{1}{\suppNumber} \sum_{\suppIx=1}^\suppNumber \delta_{\langle \mathbf{v}_\projIx, \distOnePoint_\suppIx \rangle}$}
		\STATE{\ \ $\textstyle (\Pi_{\mathbf{v}_\projIx})_\# \disttwo = \frac{1}{\suppNumber} \sum_{\suppIx=1}^\suppNumber \delta_{\langle \mathbf{v}_\projIx, \distTwoPoint_\suppIx \rangle}$}
		\STATE{Compute sorted lists of supports:}
		\STATE{\ \ $\textstyle (w_\suppIx)_{\suppIx=1}^\suppNumber \leftarrow \texttt{sort}((\langle \mathbf{v}_\projIx , \distOnePoint_\suppIx \rangle)_{\suppIx=1}^\suppNumber)$}
		\STATE{\ \ $\textstyle (z_\suppIx)_{\suppIx=1}^\suppNumber \leftarrow \texttt{sort}((\langle \mathbf{v}_\projIx , \distTwoPoint_\suppIx \rangle)_{\suppIx=1}^\suppNumber)$}
		\STATE{Compute one-dimensional Wasserstein distance:}
		\STATE{\ \ $\textstyle \wass^p_p( (\Pi_{\mathbf{v}_\projIx})_\# \distone, (\Pi_{\mathbf{v}_\projIx})_\# \disttwo)\! =\! \frac{1}{\suppNumber} \sum_{\suppIx=1}^\suppNumber |w_\suppIx - z_\suppIx|^p$
		}	
		\ENDFOR
		\RETURN
		$\frac{1}{\projNumber} \sum_{\projIx=1}^\projNumber \wass^p_p((\Pi_{\mathbf{v}_\projIx})_\# \distone,(\Pi_{\mathbf{v}_\projIx})_\# \disttwo) $
	\end{algorithmic}
\end{algorithm}


\subsection{Analysis of orthogonal couplings}

We now compare the~\textit{mean squared error} (MSE) of i.i.d.\ and orthogonal estimation of the~sliced and projected Wasserstein distance between distributions $(\distone, \disttwo) \in \mathscr{P}_{(\suppNumber)}(\mathbb{R}^d) \times \mathscr{P}_{(\suppNumber)}(\mathbb{R}^d)$. As the~MSE of an~unbiased estimator is equal to its variance, improving upon i.i.d.\ requires the~cross-covariance induced by sampled directions to be negative. This motivates us to first study a~class of \textit{stratified} estimators which is proved to be statistically superior to the~i.i.d.\ approach. The~main drawback of the~stratification scheme is its $\mathcal{O}( (\suppNumber !)^2 )$ computational complexity.

The~importance of stratified estimators for our purposes comes from the~fact that their improved accuracy is due to the~increase in average diversity of induced couplings (cf.\ \Cref{def:couplings}), a~property that is also typical for the~orthogonally coupled estimators. Orthogonal coupling can therefore be seen as a~computationally tractable approximation to stratification. As we observe in experiments, this approximation usually indeed leads to improved MSE. However, we prove that the~improvement in the~case of sliced Wasserstein estimation is not universal over all pairs of distributions $(\eta, \mu) \in \ \mathscr{P}_{p}(\mathbb{R}^d) \times \mathscr{P}_{p}(\mathbb{R}^d)$, contrary to the intuition of \citet{AutomatedColorGrading}. 

\subsubsection{Improving MSE by stratification}

We begin by formalising stratification and establishing its dominance over i.i.d.\ estimation in terms of MSE.
\begin{definition}[Stratified estimation]\label{def:stratification}
	Let $(\mathcal{X}, \mathcal{A})$ be a~measurable space, $X$ a~random variable with probability distribution $\mathrm{Law}(X)$ taking values in $\mathcal{X}$, and $f \colon \mathcal{X} \to \R{}$ an~integrable function with $\targetVal \coloneqq \E [f(X)]$. Assume $\{ A_k \}_{k=1}^{K} \subseteq \mathcal{A}$ is a~finite disjoint partition of $\mathcal{X}$. An~estimator of $\targetVal$
	\begin{equation*}
		\estim{\projNumber}
		=
		\frac{1}{\projNumber}
		\sum_{i=1}^{\projNumber}
			f(X_i)
		\, ,
	\end{equation*}
	will be called \emph{stratified} if  $\mathrm{Law}(X_i) = \mathrm{Law}(X)$, $\forall i$, and all bivariate marginals $\mathrm{Law}((X_i, X_j))$, $i \neq j$, satisfy:
	\begin{enumerate}[label=(\roman*),topsep=-1ex,itemsep=-1ex]
		\item $X_i$ and $X_j$ are conditionally independent given the~$\sigma$-algebra generated by $\{ A_k \}_{k=1}^{K}$;
		\item  $\Prob(X_i \in A_k, X_j \in A_l)$ is less than (resp.\ greater than) or equal to $\Prob(X \in A_k) \Prob(X \in A_l)$ when $k = l$ (resp.\ $k \neq l$), for any $k, l \in [K]$;
		\item $X_i$ and $X_j$ are pairwise exchangeable.
	\end{enumerate}
	The~inequality in (ii) is required to be strict for at least one $(i, j), i \neq j$, in the~$k = l$ case for some $k \in [K]$.
\end{definition}
\begin{restatable}{theorem}{mseStrat}\label{thm:mse_strat_estim}
	The~MSE of any stratified estimator is lower or equal to that of an~i.i.d.\ estimator. A~stratified estimator for which the~inequality is strict exists whenever $\exists \, k, l \in [K]$ such that $\E [f(X) \, \vert \, X \in A_k] \neq \E [f(X) \, \vert \, X \in A_l]$ and $\Prob(X \in A_k) > 0$, $\Prob(X \in A_l) > 0$.
\end{restatable}

Stratification is a well-established means of achieving variance reduction in Monte Carlo (see e.g.\ \citet{MCbook}). However, it is a particularly appealing approach in the context of sliced Wasserstein and projected Wasserstein estimation, as there is a natural partition of the space $S^{d-1}$ to consider. Bringing~\Cref{def:stratification} into the context of sliced and projected Wasserstein estimation, we take 
$\mathcal{X} = S^{\ortd - 1}$, $X_n = \mathbf{v}_n$, $\mathrm{Law} (X) = \mathrm{Unif}(S^{d-1})$, and $f$ to be the~function computed in the~inner loop of \Cref{alg:sw,alg:pw} respectively. 

Revisiting~\Cref{def:couplings}, it is natural to consider partitioning $S^{\ortd - 1}$ into sets $\{E_\tau\}_{\tau \in \mathcal{S}_\suppNumber}$, where $E_\tau \coloneqq \{\mathbf{v} \in S^{\ortd - 1} \, \vert \, \tau \in \Sigma_{\mathbf{v}} \}$ with $\Sigma_{\mathbf{v}}$ denoting the set of optimal matchings for direction $v$, and $\mathcal{S}_\suppNumber$ is the~set of all permutations of $[\suppNumber]$. These sets need not be disjoint which we amend using the~following observation: multiple couplings are optimal iff either (a) $\mathbf{x}_i = \mathbf{x}_j$ or $\mathbf{y}_i = \mathbf{y}_j$, for some $i \neq j$; or (b) $\langle \mathbf{v} , \distOnePoint_i \rangle = \langle \mathbf{v} , \distOnePoint_j \rangle$ or $\langle \mathbf{v} , \distTwoPoint_i \rangle = \langle \mathbf{v} , \distTwoPoint_j \rangle$, for some $i \neq j$. In (a), we are free to deterministically pick any of the~optimal couplings as the~contribution to both the~sliced and projected Wasserstein integrals will be the~same under any of them. The~events in (b) are then null sets and we can thus again safely pick any of the~available couplings. 

Stratification with the~modified $\{E_\tau\}_{\tau \in \mathcal{S}_\suppNumber}$ partition can therefore be applied to estimation of projected and sliced Wasserstein distances and by~\Cref{thm:mse_strat_estim} will lead to improved MSE in all but degenerate cases.

\subsubsection{Orthogonal coupling of directions as approximate stratification}

The~last section presented a~sampling scheme which renders i.i.d.\ sampling \textit{inadmissible} in terms of MSE. However, the~proposed stratification approach crucially relies on knowledge of the~$E_\sigma$ regions and our ability to sample uniformly from these. 

\begin{remark}\label{rem:half_spaces}
	Each region $E_\sigma$, can be written as a~finite union of simply-connected sets. 
	
	Specifically, a~coupling $\sigma \in \mathcal{S}_\suppNumber$ is optimal for a~given $\mathbf{v}$ iff it corresponds to the~coupling implied by associating the~projected points according to their order. The~region where a~particular fixed ordering $\tau \in \mathcal{S}_\suppNumber$ of $\{ \langle \mathbf{v} , \mathbf{x}_i \rangle \}_{i=1}^\suppNumber$ is achieved can be obtained as follows:
	\begin{enumerate}[label=(\roman*),topsep=-1ex,itemsep=-1ex]
		\item for $i = 1, \ldots , \suppNumber - 1$, define the~half-spaces
		\begin{align*}
			H_{\tau(i), \tau(i+1)}^{\mathbf{x}}  \!
			&\coloneqq \!
			\{ \mathbf{v} \in S^{\ortd - 1} \vert \, \langle \mathbf{v} , \mathbf{x}_{\tau(i)} - \mathbf{x}_{\tau(i + 1)} \rangle \! \leq 0 \}
			\, ;
		\end{align*}
		\item obtain the~region $B_\tau^{\mathbf{x}} = \bigcap_{i = 1}^{\suppNumber - 1} H_{\tau(i), \tau(i+1)}^{\mathbf{x}}$.
	\end{enumerate}
	Defining $B_\tau^\mathbf{y}$ analogously and using the~definition $E_\sigma = \{\mathbf{v} \in S^{\ortd - 1} \, \vert \, \sigma \in \Sigma_{\mathbf{v}} \}$, we can write $E_\sigma$ as a~finite union of intersections of half-spaces:
	\begin{equation*}
		E_\sigma = \bigcup_{\tau \in \mathcal{S}_\suppNumber} (
			B_{\tau}^{\mathbf{x}} \cap B_{\tau \circ \sigma}^{\mathbf{y}}
		)
		\, ,
	\end{equation*}
	where $\tau \circ \sigma$ denotes composition of the~two mappings.
\end{remark}

By \Cref{rem:half_spaces}, the~structure of $E_\sigma$ quickly grows in complexity as $\suppNumber$ increases. Practical implementation of the~algorithm is thus computationally intractable for all but very small problems.

However, we might view the~orthogonal coupling of the~$\{\mathbf{v}_n\}_{n=1}^\projNumber$ directions as an~approximation to stratification, since: (a) the~directions are pairwise exchangeable and marginally $\mathrm{Unif}(S^{d-1})$, and (b) the~orthogonal coupling of the~directions should intuitively decrease the~chance of sampling the~same induced coupling because the~individual $E_\sigma$ are finite unions of simply-connected sets (cf.\ \Cref{rem:half_spaces}). However, even if we assumed that Condition~(ii) from \Cref{def:stratification} holds, Condition~(i) will only be satisfied in the~case of the~projected Wasserstein distance for which $f(\mathbf{v})$ is piecewise constant (cf.\ \Cref{def:stratification}).

The~following result for the~simplified case $\suppNumber = 2, \ortd = 2$ supports the~above intuition, showing that orthogonal coupling improves MSE in the~projected but not necessarily in the~sliced Wasserstein case.
\begin{restatable}{proposition}{mseSWPW}\label{prop:mse_sw_pw_2x2}
	Let $\suppNumber = 2$ and $\ortd = 2$. Then~orthogonally coupled estimator of projected Wasserstein distance satisfies \Cref{def:stratification}. For the~sliced Wasserstein distance, neither i.i.d.\ nor orthogonal estimation dominates the~other in terms of MSE.
\end{restatable}

\begin{figure}[ht]
	\centering
	\setbox1=\hbox{\includegraphics[keepaspectratio,width=.6\columnwidth]{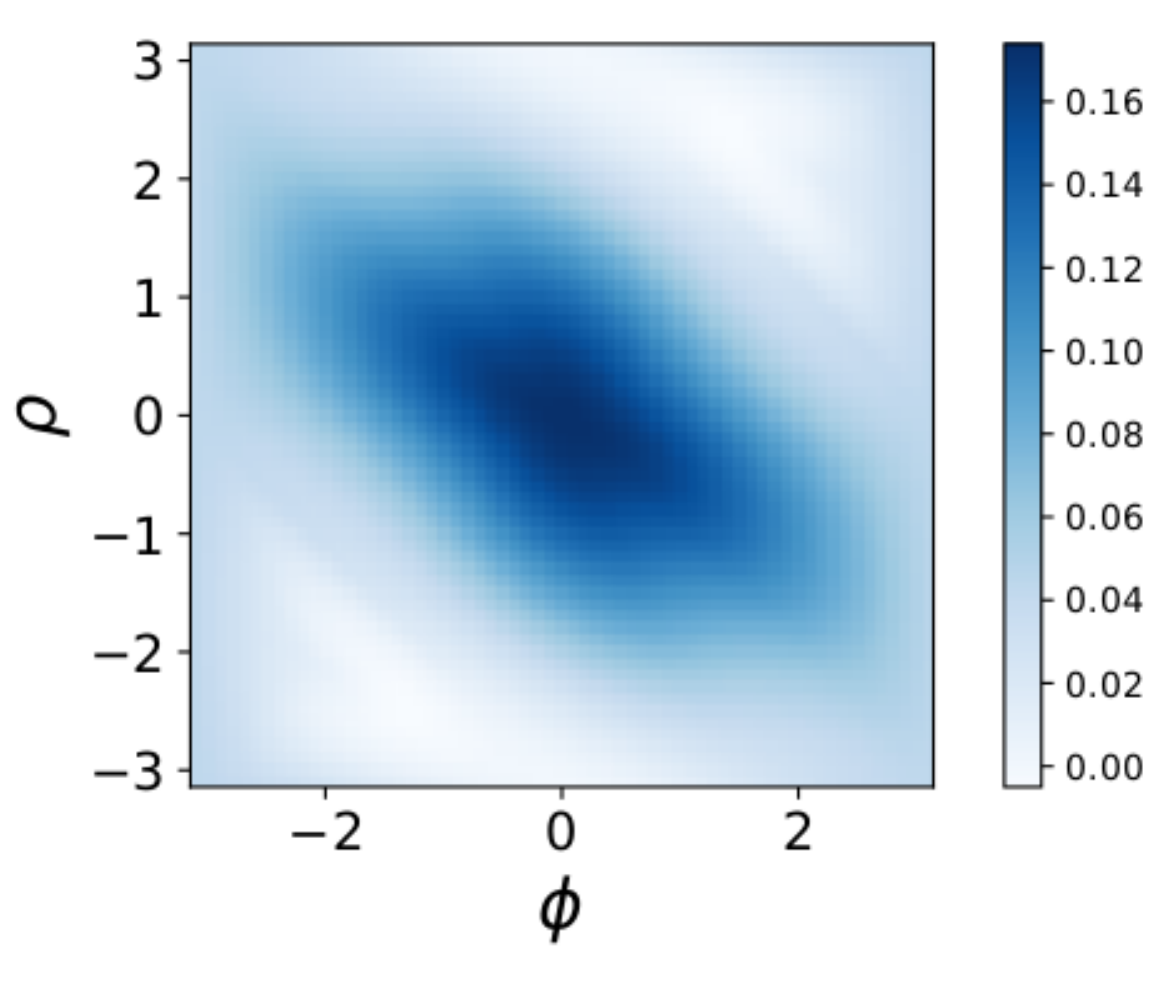}}
	\setbox2=\hbox{\includegraphics[keepaspectratio,width=.4\columnwidth]{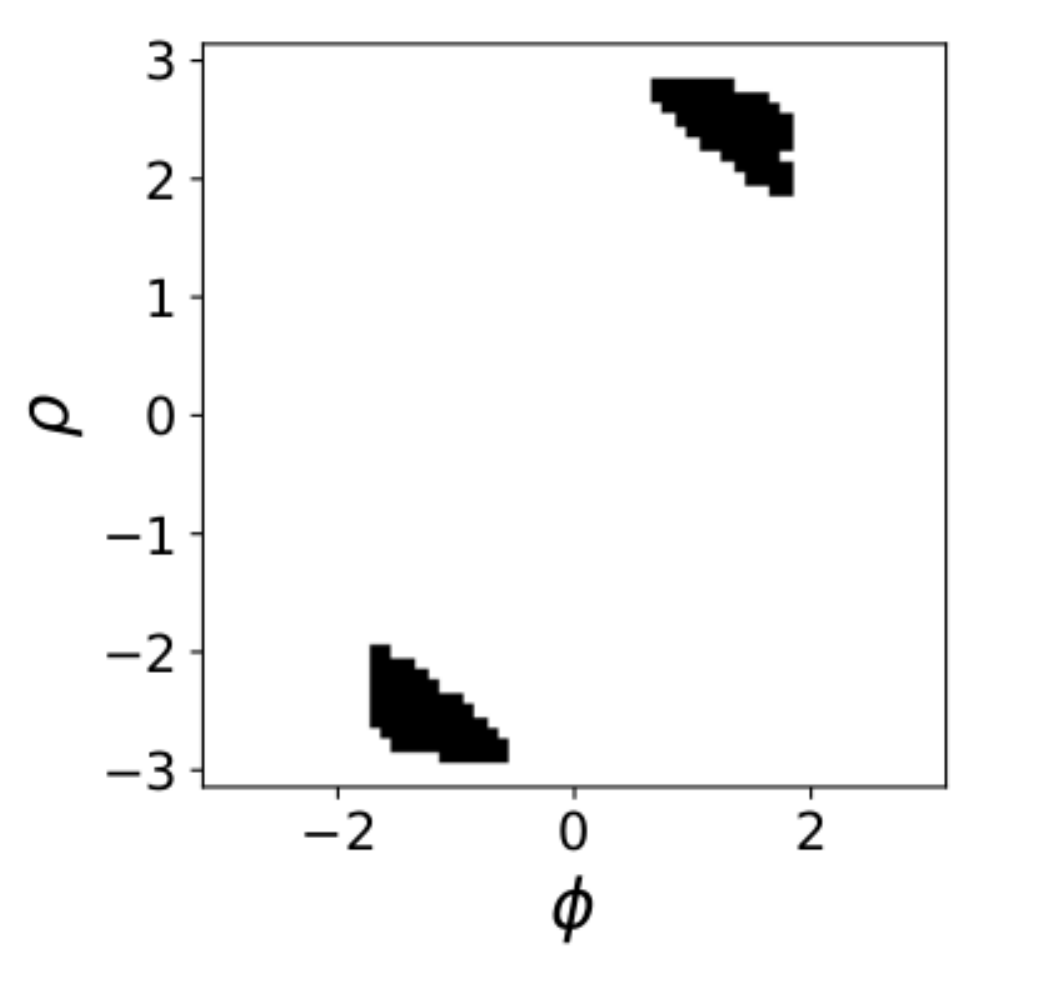}}
	\null\hfill
	\subfloat{%
		\raisebox{0.5\ht2-0.5\ht1}{\includegraphics[keepaspectratio,width=.55\columnwidth]{img/mse_2x2}}
	}
	\hfill
	\subfloat{%
		\includegraphics[keepaspectratio,width=.38\columnwidth]{img/mse_2x2_bin}} \hfill
	\hfill\null
	\caption{Difference between MSE of the~orthogonally coupled and i.i.d.\ estimator of $1$-sliced Wasserstein distance with $\projNumber = 2$, for various datasets in the~$\suppNumber = 2, \ortd = 2$ scenario.
	Each dataset is represented by a~single point.
	For all, $\mathbf{x}_1 = [0, 0]^\top$, $\mathbf{x}_2 = [1, 0]^\top$ is fixed, and $\mathbf{z} = \mathbf{x}_1 - \mathbf{y}_1$ and $\mathbf{r} = \mathbf{y}_1 - \mathbf{y}_2$ have unit norm. The~varying factors are the~angles between the~x-axis and the~vectors $\mathbf{r}$ and $\mathbf{z}$, which we respectively denote by $\phi$ and $\rho$.
	The~numerical difference between MSEs for each dataset is on the~left (the~higher the~better is orthogonal). The~two connected black regions on the~right highlight the~configurations of $\phi$ and $\rho$ for which i.i.d.\ is better than orthogonal.}
\end{figure}

\section{EXPERIMENTS}\label{sec:experiments}


We now present empirical evaluation of the theory introduced in the previous sections.

\subsection{Distance estimation}
We begin with a testbed of small-scale problems, which will aid intuition as to the advantages of orthogonal estimation. The problems we consider consist of computing
\begin{align}\label{eq:problems}
	\SW_p\left(\frac{1}{\suppNumber} \sum_{\suppIx=1}^\suppNumber \delta_{\mathbf{x}_i} , \frac{1}{\suppNumber} \sum_{\suppIx=1}^\suppNumber \delta_{\mathbf{y}_i} \right) \, .
\end{align}
To generate a collection of problems, we sample $(\mathbf{x}_\suppIx)_{\suppIx=1}^\suppNumber$, $(\mathbf{y}_\suppIx)_{\suppIx=1}^\suppNumber$ independently from distributions $N(\mu_\mathbf{x}, I)$, $N(\mu_\mathbf{y}, I)$. In Figure \ref{fig:lowdim}, we plot comparisons of the MSE achieved by estimators using i.i.d. and orthogonally-coupled projection directions for a variety of values of the parameters $d$ and $N$, in the case $p=1$, $\mu_\mathbf{x} = \mu_\mathbf{y}$, and use a number of projection direction $\projNumber$ equal to the dimensionality $d$. Each pair of sampled distributions in Expression \eqref{eq:problems} corresponds to a scatter point in the relevant graph in Figure \ref{fig:lowdim}. The results for the case $d=2$ illustrate our theoretical results that orthogonality does not always guarantee improved MSE, although in the vast majority of cases, there is indeed an improvement. The case of $d=50$ illustrates that as dimensionality increases, the improved performance of orthogonally-coupled projection directions becomes more robust.

\begin{figure}\centering
	\null\hfill
	\includegraphics[keepaspectratio, width=.49\columnwidth]{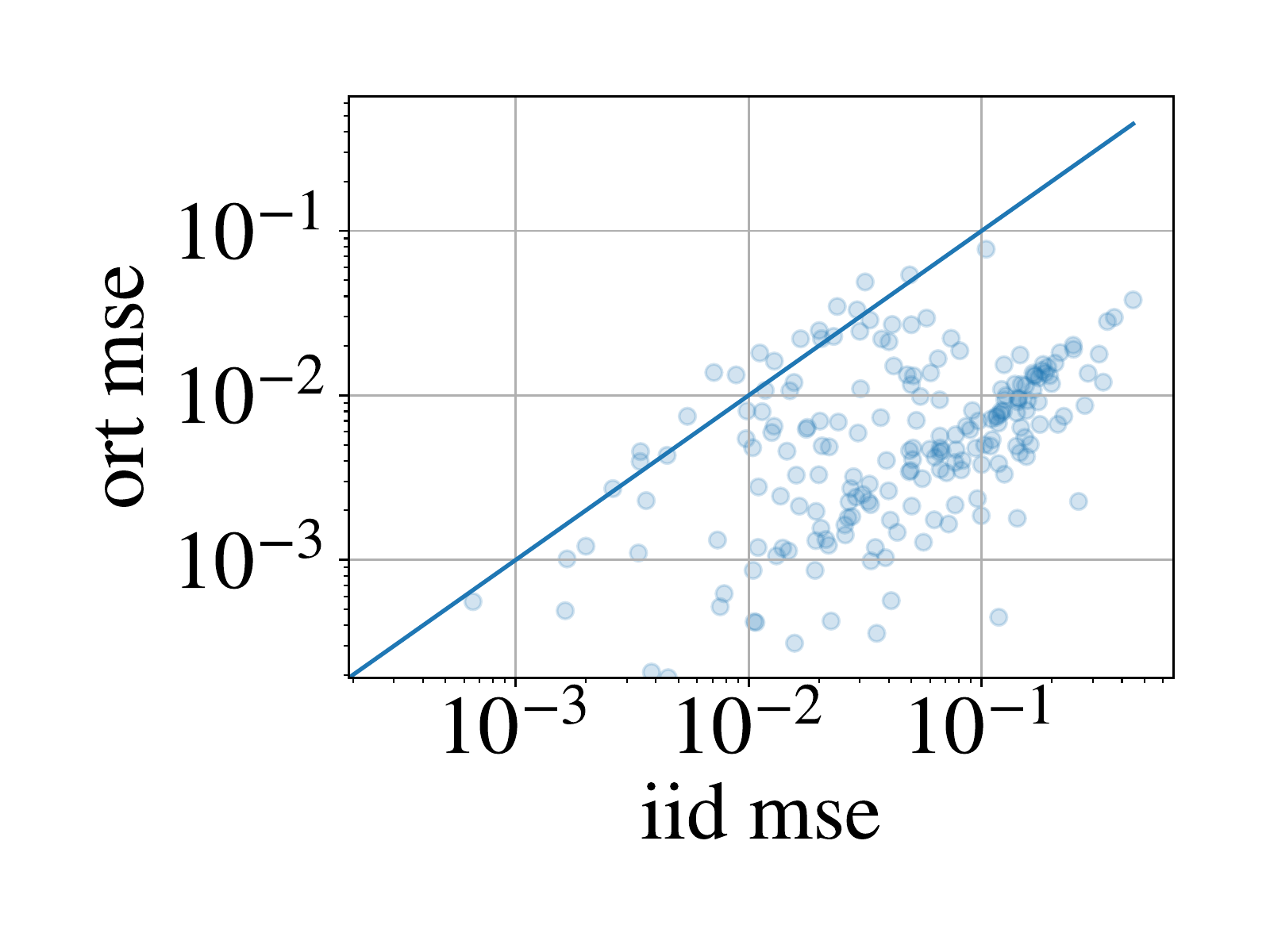}\hfill	\includegraphics[keepaspectratio, width=.49\columnwidth]{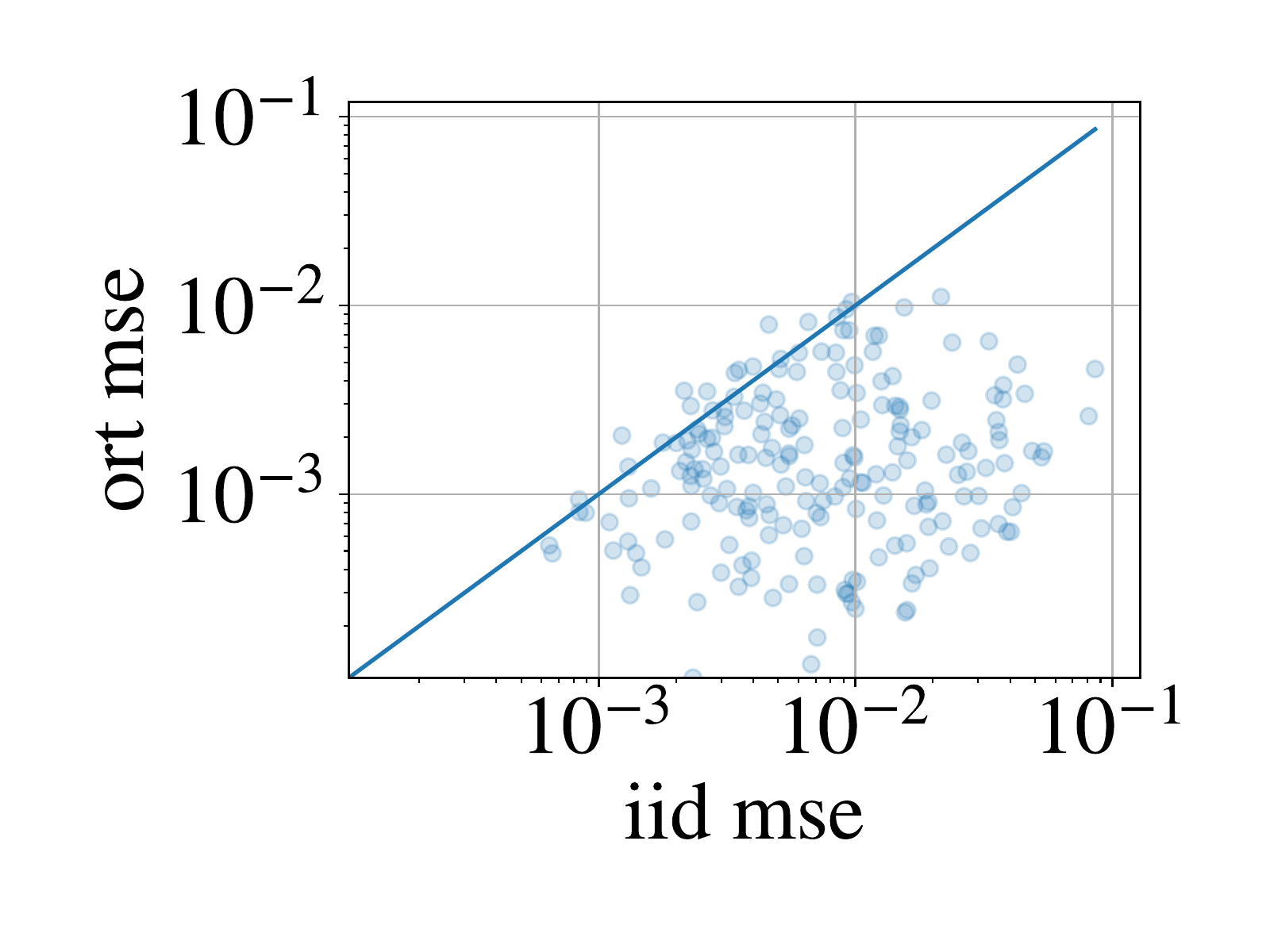}\hfill\null
	
	\vspace{-0.35cm}
	
	\null\hfill
	\includegraphics[keepaspectratio, width=.49\columnwidth]{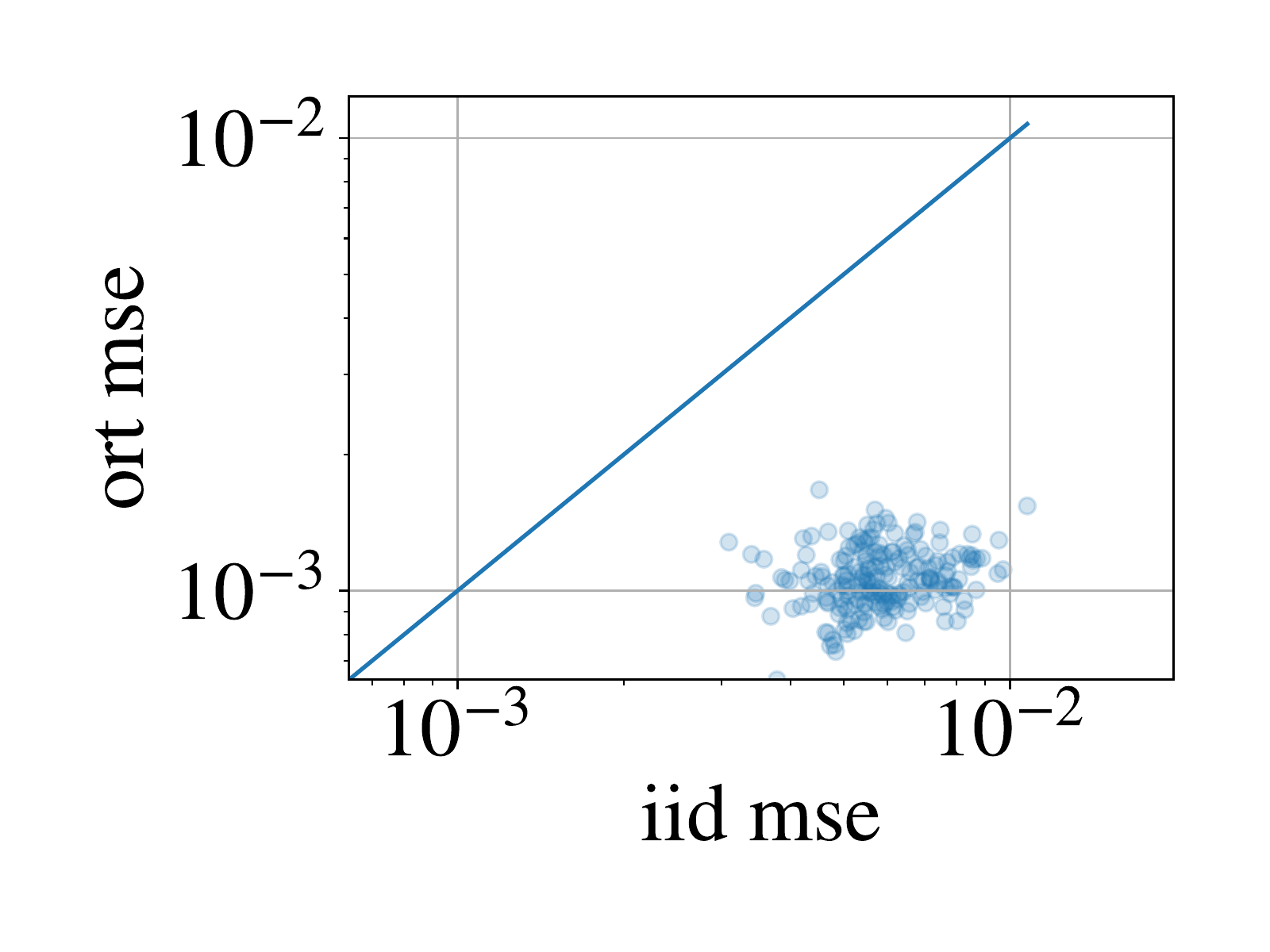}\hfill
	\includegraphics[keepaspectratio, width=.49\columnwidth]{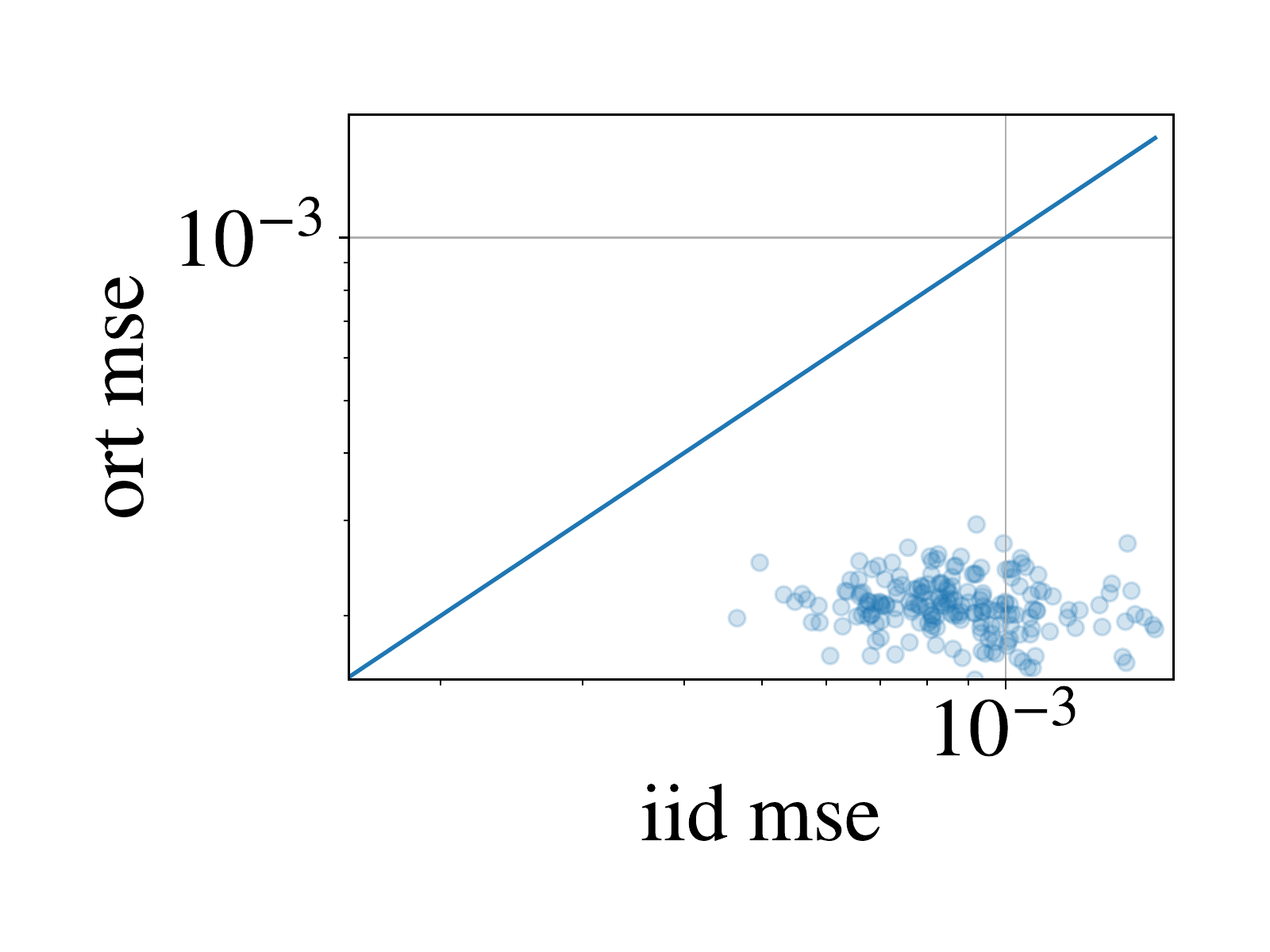}\hfill\null
	
	
	\caption{Scatter plots of orthogonal estimator MSE vs. i.i.d. estimator MSE for $d=2$ (left column) and $d=50$ (right column), and $\suppNumber=2$ (top row) and $\suppNumber=10$ (bottom row).}
	\label{fig:lowdim}
\end{figure}

\subsection{Generative modelling}\label{sec:genmod}
We study the effect of orthogonal estimation in the context of generative modeling. In particular, we study the Sliced Wasserstein Auto-Encoders \cite{swae} on MNIST dataset. The auto-encoder consists of an encoder $f_\theta(\cdot)$ with parameter $\theta$ and a decoder $g_\phi(\cdot)$ with parameter $\phi$. For an given observation $\mathbf{x} \in \mathbb{R}^n$, the encoder computes an hidden code $\mathbf{z} = f_\theta(\mathbf{x}) \in \mathbb{R}^h$. With an hidden code $\mathbf{z}$, the decoder computes a generated sample $\mathbf{\tilde{x}} = g_\phi(\mathbf{z})$. Given a distribution of $m$ observations $\{\mathbf{x}_i\}_{i=1}^m$, let $P(X) = \frac{1}{m} \sum_{i=1}^m \delta_{\mathbf{x}_i}$ be the empirical distribution, we hope to jointly train an encoder $f_\theta(\cdot)$ that uncovers the hidden structure of $P(X)$ and a decoder $g_\phi(\cdot)$ that generates samples similar to $P(X)$. Let $p_\theta(\mathbf{z})$ be the push-forward distribution from $\mathbf{z} = f_\theta(\mathbf{x}), \mathbf{x} \sim P(X)$ and $p(\mathbf{z})$ be a prior distribution over hidden codes. The loss of the auto-encoder is defined as 
\begin{align}
L_{\text{ae}}(\theta,\phi)\! =\! \mathbb{E}_{\mathbf{x} \sim P(X)}[\|g_\phi(f_\theta(\mathbf{x})) - \mathbf{x}\|]\! + \!\text{SW}_1(p_\theta(\mathbf{z}), p(\mathbf{z})) \nonumber
\end{align}
where the first term is a reconstruction error and the second term is to enforce that the generated hidden codes $p_\theta(\mathbf{z})$ be close to the prior distribution $p(\mathbf{z})$. We then update $\theta,\phi$ by approximating gradient descent $(\theta,\phi) \leftarrow (\theta,\phi) - \alpha \nabla_{(\theta,\phi)}L_\text{ae}(\theta,\phi)$ with learning rate $\alpha$. To estimate $\text{SW}_1(p_\theta(\mathbf{z}), p(\mathbf{z}))$, we compare orthogonal vs. i.i.d. Monte Carlo samples and we expect that the benefits from a more accurate estimate of the sliced Wasserstein distance translates into higher quality gradient updates. Experiment details are in the Appendix.

In Figure \ref{fig:ae}, we present the learning curves of auto-encoders using three methods: i.i.d. Monte Carlo estimate, orthogonal estimate and an approximation to orthogonal estimation using Hadamard-Rademacher (HD) random matrices (see Appendix) to approximate orthogonal estimate. We observe that the effect of orthogonality depends on the hyperparameters in the learning procedure: when the learning rate is large $\alpha = 10^{-4}$ (Figure \ref{fig:ae}, left), both estimates behave similarly; when the learning rate is small $\alpha = 10^{-5}$ (Figure \ref{fig:ae}, right), orthogonal estimate leads to a slightly faster convergence than i.i.d. estimate. When using HD matrices as a proxy to compute orthogonal estimates, we always benefit from the computational benefit at training time.

\begin{figure}\centering
	\null\hfill
	\includegraphics[keepaspectratio, width=.49\columnwidth]{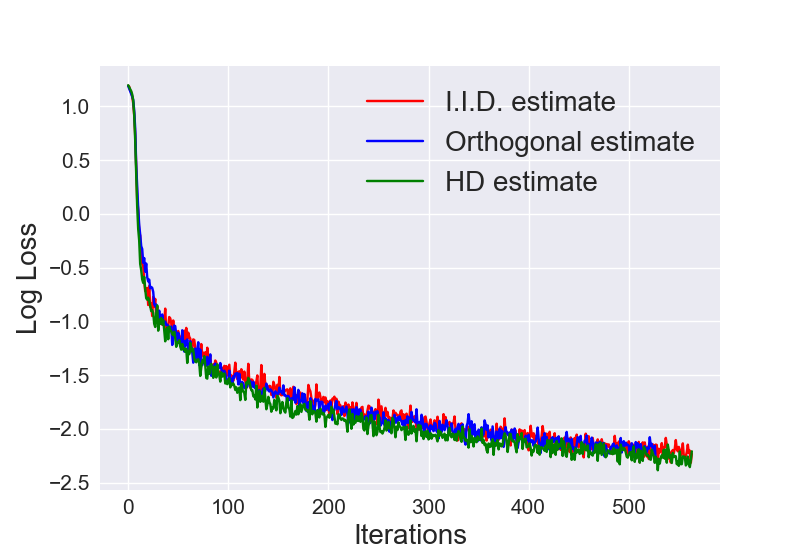}\hfill	\includegraphics[keepaspectratio, width=.49\columnwidth]{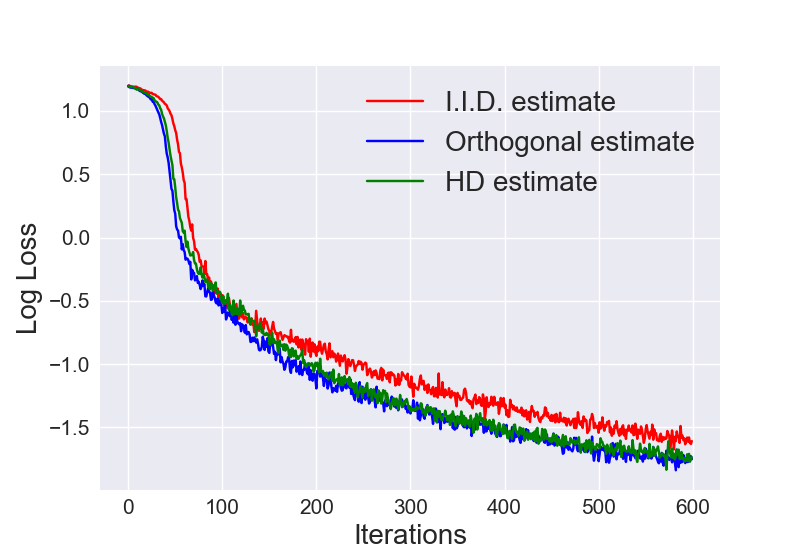}\hfill\null
	\caption{Training curves of Sliced Wasserstein Auto-encoders with three methods to compute Sliced Wasserstein distance: i.i.d. Monte Carlo estimate (red), orthogonal estimate (blue) and HD matrix for orthogonal estimate (green). Vertical axis is the log training loss, horizontal axis is the number of iterations. Left uses a learning rate of $\alpha = 1.0\cdot10^{-4}$ and right uses a learning rate of $\alpha = 1.0 \cdot 10^{-5}$.}
	\label{fig:ae}
\end{figure}

\subsection{Reinforcement learning}
In reinforcement learning (RL), at time $t$ an agent is in state $\mathbf{s}_t$, takes an action $\mathbf{a}_t$, receives an instant reward $r_t$ and transitions to next state $\mathbf{s}_{t+1}$. The objective is to search for a policy $\pi_\theta: \mathbf{s}_t \mapsto \mathbf{a}_t$ parameterized by $\theta$ such that the expected discounted cumulative reward $J(\pi_\theta) = \mathbb{E}_{\pi_\theta}[\sum_{t=0}^\infty \gamma^t r_t]$ for some discount factor $\gamma \in (0,1)$ is maximized. Policy gradient algorithms apply (an approximation of) the gradient update $\theta_{\text{new}} \leftarrow \theta_{\text{old}} + \alpha \nabla_{\theta_{\text{old}}} J(\pi_{\theta_{\text{old}}})$ to iteratively improve the policy. Trust region policy optimization \citep{schulman2015trust} requires that $D(\pi_{\theta_{\text{old}}} || \pi_{\theta_{\text{new}}}) \leq \epsilon$ for some $\epsilon > 0$ to ensure that the updates are stable, where $D(\cdot,\cdot)$ is some discrepancy measure between two policies. Previously, \citet{schulman2015trust} propose to set $D(\cdot,\cdot) = \mathbb{KL}[\cdot || \cdot]$ as the KL divergence, while \citet{zhang2018policy} set $D(\cdot,\cdot) = \text{W}_1(\cdot,\cdot)$ as the $1$-Wasserstein distance. As alternates to the discrepancy measure, we take $D(\cdot,\cdot)$ to be the sliced Wasserstein distance $\text{SW}_1(\cdot,\cdot)$ or projected Wasserstein distance $\text{PW}_1(\cdot,\cdot)$. For fast optimization, instead of constructing an explicit constraint, we adopt a penalty formulation of the trust region \citep{schulman2017proximal,zhang2018policy} and update $\theta_{\text{new}} \leftarrow \theta_{\text{old}} + \alpha \nabla_{\theta_{\text{old}}}(J(\pi_{\theta_{\text{old}}}) - \lambda D(\theta_{\text{old}},\theta_{\text{new}}))$ for some penalty constant $\lambda > 0$. We present all algorithmic and implementation details in the Appendix.

Since projected Wasserstein corrects for the implicit ``bias'' introduced by sliced Wasserstein, we expect the trust region by projected Wasserstein lead to more stable training. In Figure \ref{fig:rl}, we show the training curves on benchmark tasks HalfCheetah (right) and Hopper (left) \citep{brockman2016openai}. We compare the training curves of three schemes: no trust region (red), trust region by sliced Wasserstein distance (blue) and trust region by projected Wasserstein distance (green). In most tasks with simple dynamics as Hopper, we do not see significant difference between three methods; however, in tasks with more complex dynamics such as HalfCheetah, we observe that trust region updates with projected Wasserstein distance leads to slightly more stable updates than the other two baselines, achieving higher cumulative rewards within a fixed number of training steps.

\begin{figure}\centering
	\null\hfill
	\includegraphics[keepaspectratio, width=.49\columnwidth]{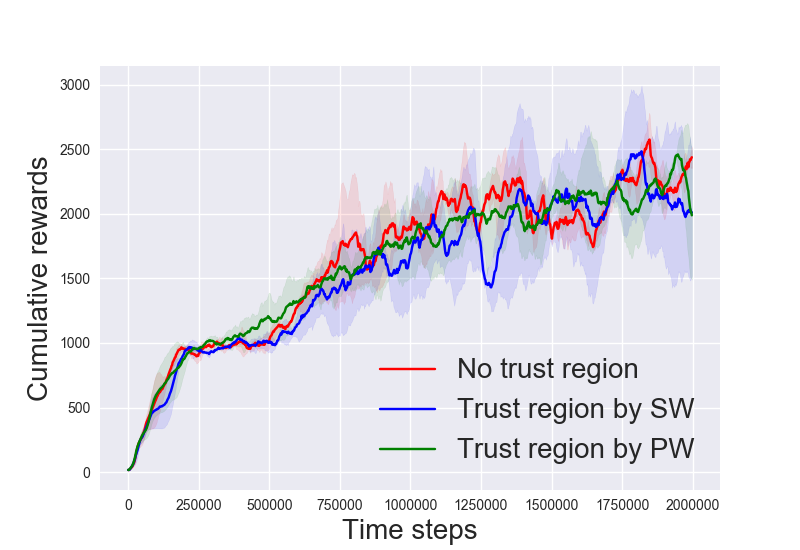}\hfill	\includegraphics[keepaspectratio, width=.49\columnwidth]{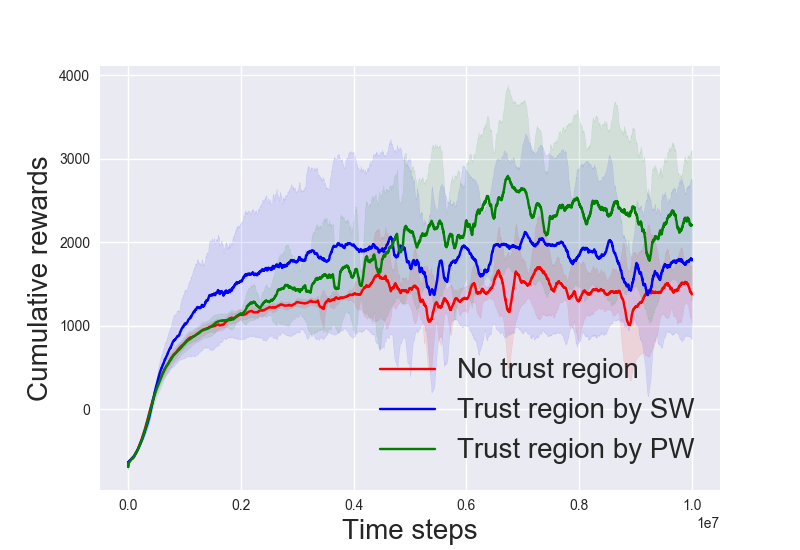}\hfill\null
	\caption{Training curves of RL with three methods to compute policy gradients updates on benchmark tasks (left: Hopper, right: HalfCheetah): no trust region (red), trust region by sliced Wasserstein (blue) and trust region by projected Wasserstein (green). Training curves show the mean $\pm$ std performance across 5 random seeds. Vertical axis is the cumulative reward, horizontal axis is the number of time steps. }
	\label{fig:rl}
\end{figure}


\section{CONCLUSION}\label{sec:conclusion}


We have considered projected Wasserstein distance, a variant of sliced Wasserstein distance, and studied orthogonal couplings of projection directions in estimators of sliced and projected Wasserstein distances. In doing so, we have also given an interpretation of orthogonal coupling as an efficient, approximate means of performing stratified sampling. Our empirical evaluations show that orthogonality can dramatically reduce estimator variance, and these benefits are translated over to downstream tasks such as generative modelling in certain circumstances. 
Important areas for future work include deepening our understanding of the relationship between improvements in estimation of Wasserstein distances themselves and improvements in downstream tasks such as distribution learning, and strengthening our understanding of the effectiveness of orthogonal couplings in Monte Carlo estimators of Wasserstein distances.

\subsubsection*{Acknowledgements}

MR acknowledges support by EPSRC grant
EP/L016516/1 for the Cambridge Centre for Anal-
ysis. JH acknowledges support by a Nokia CASE
Studentship. YHT acknowledges the cloud credits
provided by Amazon Web Services. AW acknowledges
support from the David MacKay Newton research fel-
lowship at Darwin College, The Alan Turing Institute
under EPSRC grant EP/N510129/1 \& TU/B/000074,
and the Leverhulme Trust via the CFI.

\bibliographystyle{apalike}
\bibliography{00_sliced_wasserstein}

\newpage
\onecolumn

\section*{APPENDIX: Orthogonal Estimation of Wasserstein Distances}

\section{Proofs of results in Section \ref{sec:projwasserstein}}


\propMetric*
\begin{proof}
	Symmetry and non-negativity are immediate. We thus turn our attention to proving: (i) $\PW_p(\eta, \mu) = 0$ iff $\eta = \mu$; and (ii) the triangle inequality.
	
	For (i), first let $\distone = \frac{1}{\suppNumber} \sum_{\suppIx = 1}^{\suppNumber} \delta_{\distOnePoint_\suppIx}$ and $\disttwo = \frac{1}{\suppNumber} \sum_{\suppIx = 1}^{\suppNumber} \delta_{\distTwoPoint_\suppIx}$ be distinct. Then for \emph{any} bijective map $\sigma : [\suppNumber] \rightarrow [\suppNumber]$, we have $\sum_{\suppIx =1}^{\suppNumber} \| \distOnePoint_\suppIx - \distTwoPoint_{\sigma(\suppIx)}\|_2^p > 0$, and hence immediately we have $\PW(\distone, \disttwo) > 0$. The converse direction is clear.
	
	For (ii), let $\distone = \frac{1}{\suppNumber} \sum_{\suppIx = 1}^{\suppNumber} \delta_{\distOnePoint_\suppIx}$, $\disttwo = \frac{1}{\suppNumber} \sum_{\suppIx = 1}^{\suppNumber} \delta_{\distTwoPoint_\suppIx}$, and $\distthree = \frac{1}{\suppNumber} \sum_{\suppIx = 1}^{\suppNumber} \delta_{\distThreePoint_\suppIx}$. Fix $\mathbf{v} \in S^{d-1}$, and without loss of generality, assume that the points $(\distOnePoint_\suppIx)_{\suppIx=1}^\suppNumber$, $(\distTwoPoint_\suppIx)_{\suppIx=1}^\suppNumber$, 
	$(\distThreePoint_\suppIx)_{\suppIx=1}^\suppNumber$ are indexed so that
	\begin{align}\label{eq:goodordering}
		\langle \mathbf{v}, \distOnePoint_1 \rangle \leq \langle \mathbf{v}, \distOnePoint_2 \rangle \leq \cdots \leq \langle \mathbf{v}, \distOnePoint_\suppNumber \rangle \, , \ \
		\langle \mathbf{v}, \distTwoPoint_1 \rangle \leq \langle \mathbf{v}, \distTwoPoint_2 \rangle \leq \cdots \leq \langle \mathbf{v}, \distTwoPoint_\suppNumber \rangle \, , \ \
		\langle \mathbf{v}, \distThreePoint_1 \rangle \leq \langle \mathbf{v}, \distThreePoint_2 \rangle \leq \cdots \leq \langle \mathbf{v}, \distThreePoint_\suppNumber \rangle \, .
	\end{align}
	Now observe that with this indexing notation, the value of the integrand in the definition of projected Wasserstein distances $\PW_p(\distone, \disttwo)$, $\PW_p(\distone, \distthree)$, and $\PW_p(\disttwo, \distthree)$  (Equation \eqref{eq:pw}) for this particular projection vector $\mathbf{v}$ are
	\begin{align}
		\frac{1}{\suppNumber}\sum_{\suppIx=1}^\suppNumber
			\|\distOnePoint_\suppIx - \distTwoPoint_\suppIx\|_2^p
			\, , \quad 
		\frac{1}{\suppNumber}\sum_{\suppIx=1}^\suppNumber
			\|\distOnePoint_\suppIx - \distThreePoint_\suppIx\|_2^p
		  \, , \quad
		\frac{1}{\suppNumber}\sum_{\suppIx=1}^\suppNumber
			\|\distTwoPoint_\suppIx - \distThreePoint_\suppIx\|_2^p
		\, ,
	\end{align}
	respectively. Thus, the full projected Wasserstein distances may be expressed as follows:
	\begin{align}
		\PW_p(\distone, \disttwo)& =
		\left\lbrack
			\sum_{(\sigma, \tau, \pi) \in \symgroup{\suppNumber}^3}
			  q(\sigma, \tau, \pi)
				\sum_{\suppIx = 1}^\suppNumber
				  \|\distOnePoint_{\sigma(\suppIx)} - \distTwoPoint_{\tau(\suppIx)}\|_2^p
		\right\rbrack^{1/p} \, , \\
		\PW_p(\distone, \distthree)& =
		\left\lbrack
			\sum_{(\sigma, \tau, \pi) \in \symgroup{\suppNumber}^3}
			  q(\sigma, \tau, \pi)
			    \sum_{\suppIx = 1}^\suppNumber
			      \|\distOnePoint_{\sigma(\suppIx)} - \distThreePoint_{\pi(\suppIx)}\|_2^p
		\right\rbrack^{1/p} \, , \\
		\PW_p(\disttwo, \distthree)& =
		\left\lbrack
			\sum_{(\sigma, \tau, \pi) \in \symgroup{\suppNumber}^3}
			  q(\sigma, \tau, \pi)
			    \sum_{\suppIx = 1}^\suppNumber
				  \|\distTwoPoint_{\tau(\suppIx)} - \distThreePoint_{\pi(\suppIx)}\|_2^p
		\right\rbrack^{1/p} \, ,
	\end{align}
	where $\sigma, \tau, \pi \in \symgroup{\suppNumber}$ are the permutations needed to re-index $(\distOnePoint_\suppIx)_{\suppIx=1}^\suppNumber$, $(\distTwoPoint_\suppIx)_{\suppIx=1}^\suppNumber$, and $(\distThreePoint_\suppIx)_{\suppIx=1}^\suppNumber$, respectively, so that Equation \eqref{eq:goodordering} holds, and $q(\sigma, \tau, \pi)$ is the probability that permutations $\sigma, \tau, \pi$ are required, given that $\mathbf{v}$ is drawn from $\mathrm{Unif}(S^{d-1})$. With these alternative expressions established, the triangle inequality for $\PW_p$ now follows from the standard Minkowski inequality.
\end{proof}

\propIneq*
\begin{proof}
	The inequality between sliced Wasserstein and Wasserstein distances is well-known, and a short proof is given by e.g. \citet{Bonnotte}. For the inequality between Wasserstein and projected Wasserstein distances, write $\distone = \frac{1}{\suppNumber} \sum_{\suppIx = 1}^{\suppNumber} \delta_{\distOnePoint_\suppIx}$ and $\disttwo = \frac{1}{\suppNumber} \sum_{\suppIx = 1}^{\suppNumber} \delta_{\distTwoPoint_\suppIx}$. Now note that
	\begin{align}
		\PW^p_p(\distone, \disttwo)
		& = 
		\mathbb{E}_{\mathbf{v} \sim \mathrm{Unif}(S^{d-1})}
		\left\lbrack\frac{1}{\suppNumber} 				
		\sum_{\suppIx=1}^\suppNumber \|\mathbf{x}_\suppIx - \mathbf{y}_{\sigma_{\mathbf{v}}(\suppIx)} \|_2^p
		\right\rbrack \\
		&\geq
		\mathbb{E}_{\mathbf{v} \sim \mathrm{Unif}(S^{d-1})}
		\left\lbrack
		\min_{\sigma \in \symgroup{\suppNumber}}
		\frac{1}{\suppNumber} 				
		\sum_{\suppIx=1}^\suppNumber \|\mathbf{x}_\suppIx - \mathbf{y}_{\sigma(\suppIx)} \|_2^p
		\right\rbrack  \\
		&= 
		\min_{\sigma \in \mathcal{S}_\suppNumber}
		\frac{1}{\suppNumber} 				
		\sum_{\suppIx=1}^\suppNumber \|\mathbf{x}_\suppIx - \mathbf{y}_{\sigma(\suppIx)} \|_2^p
		 \\
		& =
		\wass^p_p(\distone, \disttwo) \, ,
	\end{align}
	where $\mathcal{S}_\suppNumber$ is the symmetric group, i.e. the space of bijective mappings from $[\suppNumber]$ to itself.
\end{proof}


\section{Additional material relating to Section \ref{sec:ortestimation}}

\subsection{Orthogonal projected Wasserstein estimation}

We present the full algorithm applying orthogonal projection directions to estimation of the projected Wasserstein distance in Algorithm \ref{alg:pw-ort}

\begin{algorithm}
	\caption{Projected Wasserstein estimation}
	\label{alg:pw-ort}
	\begin{algorithmic}[1]
		\REQUIRE $\eta = \frac{1}{\suppNumber} \sum_{\suppIx=1}^\suppNumber \delta_{\distOnePoint_\suppIx}$, \ $\mu = \frac{1}{\suppNumber} \sum_{\suppIx=1}^\suppNumber \delta_{\distTwoPoint_\suppIx}$
		\STATE \textcolor{red}{Sample $(\mathbf{v}_\projIx)_{\projIx=1}^\projNumber \sim \mathrm{UnifOrt}(S^{d-1};N)$}
		\FOR{ $\projIx=1$ \TO $\projNumber$}
		\STATE {Compute projected distributions:}
		\STATE{\ \ $\textstyle (\Pi_{\mathbf{v}_\projIx})_\# \distone = \frac{1}{\suppNumber} \sum_{\suppIx=1}^\suppNumber \delta_{\langle \mathbf{v}_\projIx, \distOnePoint_\suppIx \rangle}$}
		\STATE{\ \ $\textstyle (\Pi_{\mathbf{v}_\projIx})_\# \disttwo = \frac{1}{\suppNumber} \sum_{\suppIx=1}^\suppNumber \delta_{\langle \mathbf{v}_\projIx, \distTwoPoint_\suppIx \rangle}$}
		\STATE{Compute optimal matching for projected distributions:}
		\STATE{\ \ $\sigma_{\mathbf{v}_\projIx} \leftarrow \texttt{argsort}((\langle \mathbf{v}_\projIx , \distOnePoint_\suppIx \rangle)_{\suppIx=1}^\suppNumber, (\langle \mathbf{v}_\projIx , \distTwoPoint_\suppIx \rangle)_{\suppIx=1}^\suppNumber)$}
		\STATE{Compute contribution from coupling:}
		\STATE{\ \ $\textstyle \frac{1}{\suppNumber} \sum_{\suppIx=1}^\suppNumber \|\distOnePoint_\suppIx - \distTwoPoint_{\sigma_{\mathbf{v}_\projIx}(\suppIx)}\|^p $
		}
		\ENDFOR
		\RETURN
		$\widehat{\PW}_p^p(\distone, \disttwo) = \frac{1}{\projNumber} \sum_{\projIx=1}^\projNumber  \frac{1}{\suppNumber}\! \sum_{\suppIx=1}^\suppNumber \|\distOnePoint_\suppIx - \distTwoPoint_{\sigma_{\mathbf{v}_\projIx}(\suppIx)}\|^p $
	\end{algorithmic}
\end{algorithm}

\subsection{Sampling from $\mathrm{UnifOrt}(S^{d-1}; N)$}\label{sec:exactsample}

As described in Definition \ref{def:unifort}, the primary task in sampling from $\mathrm{UnifOrt}(S^{d-1}; N)$ is sampling an orthogonal matrix from Haar measure on the orthogonal group $\ortgroup{d}$. This is a well-studied problem (see e.g. \cite{GenzOrtMatrices}), and we briefly review a method for exact simulation. Algorithm \ref{alg:gs} generates such matrices, and can be understood as follows. Initially, the rows of $\mathbf{A}$ are independent with uniformly random directions. Normalising and performing Gram-Schmidt orthogonalisation results in an ordered set of unit vectors that are uniformly distributed on the Steifel manifold, and hence the matrix obtained by taking these vectors as rows is distributed according to Haar measure on the orthogonal group $\ortgroup{d}$.

\begin{algorithm}
	\caption{Gram-Schmidt orthogonal matrix generation}
	\label{alg:gs}
	\begin{algorithmic}[1]
		\STATE Sample $(\mathbf{A}_{ij})_{i,j=1}^d \overset{\mathrm{i.i.d.}}{\sim} N(0,1)$
		\STATE Normalise the norms of the rows of $\mathbf{A}$ to 1.
		\STATE Perform Gram-Schmidt orthogonalisation on the rows of $\mathbf{A}$ to obtain $\widetilde{\mathbf{A}}$
		\RETURN $\widetilde{\mathbf{A}}$
	\end{algorithmic}
\end{algorithm}

\subsection{Approximate sampling from $\mathrm{UnifOrt}(S^{d-1}; N)$}

The Gram-Schmidt subroutine described in Section \ref{sec:exactsample} has computational cost $\mathcal{O}(d^3)$. Whilst in general, this cost would be dominated by the cost of computing a full Wasserstein distance between point clouds (costing $\mathcal{O}(M^{5/2} \log M)$ in the special case of matching, and at least $\mathcal{O}(M^4)$ more generally), it is desirable to reduce the cost of sampling from $\mathrm{UnifOrt}(S^{d-1}; N)$ further, to make projected/sliced Wasserstein estimation more computationally efficient. A variety of methods for \emph{approximately} sampling from $\mathrm{UnifOrt}(S^{d-1}; N)$ at a cost of $\mathcal{O}(d^2\log d)$ exist (see for example \citep{GenzOrtMatrices,HD1,HD2}), reducing the cost to approximately that of sampling independent projection directions (i.e. $\mathcal{O}(d^2)$). In our experiments, we use Hadamard-Rademacher random matrices to this end; further details are given in Section \ref{sec:hd}.


\subsection{Proofs}

\mseStrat*
\begin{proof}
	Recalling the~notation of~\Cref{def:stratification}, the~MSE of any unbiased estimator is equal to its variance
	\begin{equation*}
		\mathbb{V}(\estim{\projNumber})
		= 
		\frac{\mathbb{V}(f(X))}{\projNumber}
		+
		\frac{1}{\projNumber^2}
		\sum_{i=1}^{\projNumber}
		\sum_{j \neq i}
			\E [f(X_i) f(X_j)]
			-
			\{
				\E [f(X)]
			\}^2
		\, .
	\end{equation*}
	The~latter term on the~r.h.s.\ of the~above equation is equal to zero for the~i.i.d.\ estimator and thus MSE can only be improved if it is negative. By~\Cref{def:stratification}, $\E [f(X_i) f(X_j) \, \vert \, X_i \in A_k, X_j \in A_l] = \E [f(X) \, \vert \, X \in A_k] \E [f(X) \, \vert \, X \in A_l]$ whenever $i \neq j$. We can thus rewrite the~cross covariance as
	\begin{equation*}
		\E [f(X_i) f(X_j)]
		-
		\{
			\E [f(X)]
		\}^2
		=
		\sum_{k=1}^{K}
		\sum_{l=1}^{K}
			(p_{k, l}^{(i, j)} - p_k p_l) s_k s_l
		\, ,
	\end{equation*}
	where $p_{k, l}^{(i, j)} \coloneqq \Prob(X_i \in A_k , X_j \in A_l)$, $p_k = \Prob(X \in A_k)$, and $s_k \coloneqq \E [f(X) \, \vert \, X \in A_k]$. Defining the~matrix $[\mathbf{P}^{(i, j)}]_{k,l} \coloneqq p_{k, l}^{(i, j)}$ and the~vector $[\mathbf{p}]_k \coloneqq p_k$ we have that the~cross-covariance is non-positive for all integrable $f$ iff $\mathbf{P}^{(i, j)} - \mathbf{p} \mathbf{p}^\top$ is negative semi-definite. Observing that the~constraints on bivariate marginals in \Cref{def:stratification} ensure that each $\mathbf{P}^{(i, j)}$ is a~diagonally dominant Hermitian matrix with non-positive entries on the~main diagonal, implying negative semi-definiteness.
	
	To prove existence of a~stratified estimator which strictly improves upon i.i.d., consider the~matrix $\mathbf{P}^{(1, 2)}$ and let $k, l \in [K]$ be the~indices for which $\E [f(X) \, \vert \, X \in A_k] \neq \E [f(X) \, \vert \, X \in A_l]$ and $\Prob(X \in A_k) > 0$, $\Prob(X \in A_l) > 0$. Equate $\mathbf{P}^{(1, 2)} = \mathbf{p} \mathbf{p}^\top$ except for setting  $\mathbf{P}_{k,k}^{(1, 2)} = p_k p_k - \varepsilon$, $\mathbf{P}_{l,l}^{(1, 2)} = p_l p_l - \varepsilon$, and $\mathbf{P}_{k,l}^{(1, 2)} = \mathbf{P}_{k,l}^{(1, 2)} = p_k p_l + \varepsilon$, for some $\varepsilon > 0$ which preserves non-negativity of the~entries of $\mathbf{P}^{(1, 2)}$. If $X_1, X_2$ are sampled independently given $\{A_k\}_{k=1}^K$, and $X_{3}, \ldots, X_\projNumber$ i.i.d.\ $\mathrm{Unif}(S^{\ortd - 1})$ (if $\projNumber > 2$), then $\E [f(X_i) f(X_j)] - \{\E [f(X)]\}^2 < 0$.
\end{proof}

\mseSWPW*
\begin{proof}
	We begin by observing that for $\ortd = 2$, $\mathbf{v} \in S^{\ortd - 1}$ can be parametrised by single parameter $\phi \in [0, 2 \pi)$ as $\mathbf{v} = [\cos(\phi), \sin(\phi)]^\top \in \R{2}$. Denoting $\{\sigma_A , \sigma_B \} = \mathcal{S}_2$ with $\sigma_A(i) = i, i=1,2$ and $\sigma_B(1) = 2, \sigma_B(2) = 1$, we can characterise the~sets $\widetilde{E}_A = \{ \phi \in \R{} \, \vert \, \sigma_A \in \Sigma_{\mathbf{v}} \}$ and $\widetilde{E}_B = \{ \phi \in \R{} \, \vert \, \sigma_B \in \Sigma_{\mathbf{v}} \}$ as follows: a~matching is optimal iff it agrees with the~ordering of $\langle \mathbf{v}, \mathbf{x}_1 \rangle, \langle \mathbf{v}, \mathbf{x}_2 \rangle$ and $\langle \mathbf{v}, \mathbf{y}_1 \rangle, \langle \mathbf{v}, \mathbf{y}_2 \rangle$; therefore we can define
	\begin{align}\label{eq:half_spaces}
	\begin{aligned}
		&H_{\mathbf{x}}^+ = \{\phi \in \R{} \, \vert \, \langle \mathbf{v} , \mathbf{x}_1 - \mathbf{x}_2 \rangle \geq 0\}  \, , 
		&&H_{\mathbf{x}}^- = \{\phi \in \R{} \, \vert \, \langle \mathbf{v}, \mathbf{x}_1 - \mathbf{x}_2 \rangle \leq 0\} \, , \\
		&H_{\mathbf{y}}^+ = \{\phi \in \R{} \, \vert \, \langle \mathbf{v}, \mathbf{y}_1 - \mathbf{y}_2 \rangle \geq 0\}  \, , 
		&&H_{\mathbf{y}}^- = \{\phi \in \R{} \, \vert \, \langle \mathbf{v}, \mathbf{y}_1 - \mathbf{y}_2 \rangle \leq 0\} \, ,
	\end{aligned}
	\end{align} 
	and observe $\widetilde{E}_A \coloneqq (H_\mathbf{x}^+ \cap H_\mathbf{y}^+) \cup (H_\mathbf{x}^- \cap H_\mathbf{y}^-)$ and $\widetilde{E}_B \coloneqq (H_\mathbf{x}^+ \cap H_\mathbf{y}^-) \cup (H_\mathbf{x}^- \cap H_\mathbf{y}^+)$ As argued in the~main text, $\{\phi \in \R{} \, \vert \, \langle \mathbf{v}, \mathbf{x}_1 - \mathbf{x}_2 \rangle = 0 \} \cap [0, 2 \pi)$ and $\{\phi \in \R{} \, \vert \, \langle \mathbf{v}, \mathbf{y}_1 - \mathbf{y}_2 \rangle = 0 \} \cap [0, 2 \pi)$ are null events except for the~degenerate case when $\mathbf{x}_1 = \mathbf{x}_2$ or $\mathbf{y}_1 =\mathbf{y}_2$ for which both couplings are equivalent in terms of transportation cost, and thus we can safely treat $\{ \widetilde{E}_A \cap [0, 2\pi), \widetilde{E}_B \cap [0, 2\pi) \}$ as a~disjoint partition of $[0, 2 \pi)$, selecting a~single coupling deterministically if both are optimal.
	
	Observe that $| \langle \mathbf{v}, \mathbf{x}_i - \mathbf{y}_j \rangle| = | \langle - \mathbf{v}, \mathbf{x}_i - \mathbf{y}_j \rangle|$ and thus $\mathbf{v}$ and $-\mathbf{v}$ always induce the~same optimal couplings. This means that orthogonal coupling of $\mathbf{v}_1 = [\cos (\phi_1), \sin(\phi_1)]^\top$ and $\mathbf{v}_2 = [\cos(\phi_2) , \sin (\phi_2)]^\top$ is equivalent to setting $\phi_2 = \phi_1 \pm \tfrac{\pi}{2}$, which means both of the~orthogonal vectors induce the~same set of optimal couplings. Therefore 
	\begin{align*}
		\Prob( (\phi_1, \phi_2) \in \widetilde{E}_k \times \widetilde{E}_l \cap [0 , 2 \pi)^2) 
		&= \tfrac{1}{2} \Prob(\phi_1 \in \widetilde{E}_k \cap \{\widetilde{E}_l + \tfrac{\pi}{2} \} \cap [0 , 2 \pi)) + \tfrac{1}{2} \Prob(\phi_1 \in \widetilde{E}_k \cap \{\widetilde{E}_l - \tfrac{\pi}{2} \} \cap [0 , 2 \pi))
		\\
		&= \phantom{\tfrac{1}{2}} \Prob(\phi_1 \in \widetilde{E}_k \cap \{\widetilde{E}_l + \tfrac{\pi}{2} \} \cap [0 , 2 \pi))
		\, ,
	\end{align*}
	as $\widetilde{E}_k \cap \{\widetilde{E}_l + \tfrac{\pi}{2} \} = \widetilde{E}_k \cap \{\widetilde{E}_l - \tfrac{\pi}{2} \} $, for any $k, l \in \{A, B\}$. Because $\mathbf{v} \sim \mathrm{Unif}(S^1)$ is equivalent to $\phi \sim \mathrm{Unif}([0, 2\pi))$,
	\begin{align*}
		&\Prob(\phi \in \widetilde{E}_A \cap \{\widetilde{E}_A + \tfrac{\pi}{2} \} \cap [0, 2 \pi)) = 2 [(p - \tfrac{1}{2}) \vee 0 ] \, ,
		\\
		&\Prob(\phi \in \widetilde{E}_B \cap \{ \widetilde{E}_B + \tfrac{\pi}{2} \} \cap [0, 2 \pi)) = 2 [(\tfrac{1}{2} - p) \vee 0] \, ,
		\\
		&\Prob(\phi \in \widetilde{E}_A \cap \{ \widetilde{E}_B + \tfrac{\pi}{2} \} \cap [0, 2 \pi)) = \Prob(\phi \in \widetilde{E}_B \cap \{ \widetilde{E}_A + \tfrac{\pi}{2} \} \cap [0, 2 \pi) ) = \tfrac{1}{2} - |\tfrac{1}{2} - p|
		\, ,
	\end{align*}
	with $p \coloneqq \Prob(\phi \in \widetilde{E}_A)$. The~above equations combined with the~definition of orthogonal sampling and the~piecewise constant character of $f(\mathbf{v})$ in the~case of the~orthogonal estimation of the~projected Wasserstein distance implies that all conditions of~\Cref{def:stratification} are satisfied, proving the~first part of our proposition.
	
	Turning to estimation of the~sliced Wasserstein distance, we will reduce the~computation of the~MSE for both the~i.i.d.\ and the~orthogonal case to analytically solvable integrals, and use those to find examples of datasets for which either i.i.d.\ or orthogonal estimation is superior to the~other. First, the~expectation
	\begin{equation*}
		\E[
			\wass^p_p((\Pi_{\mathbf{v}})_\# \distone,(\Pi_{\mathbf{v}})_\# \disttwo)
		]
		=
		\frac{1}{2}
		\sum_{k \in \{A , B\}}
			\Prob(\mathbf{v} \in E_k) \,
			\E[
				|\langle \mathbf{v}, \mathbf{x}_1  - \mathbf{y}_{\sigma_k(1)} \rangle|^p
				+
				|\langle \mathbf{v}, \mathbf{x}_2  - \mathbf{y}_{\sigma_k(2)} \rangle|^p
				\, \vert \, \mathbf{v} \in E_k
			]
		\, ,
	\end{equation*}
	can be solved by using the~harmonic addition identity $\langle \mathbf{v}, \mathbf{z} \rangle = \| \mathbf{z} \| \cos(\phi - \rho)$ with $\rho$ the~angle between $\mathbf{z}$ and the~x-axis $[1, 0]^\top \in \R{2}$, for $\mathbf{v} = [\cos(\phi), \sin(\phi)]^\top$ and any $\mathbf{z} \in \R{2}$. Evaluation of the~above expectation then reduces to computation of a~weighted sum of integrals of the~form $\int_{\widetilde{E}_k \cap [0, 2\pi)} |\! \cos(\phi - \rho)|^p \mathrm{d}\phi$ which can be solved using basic identities. The~approach is analogous for the~second moment, with the~only difference being that the~integrals will be of the~form $\int_{\widetilde{E}_k \cap [0, 2 \pi)} |\! \cos(\phi - \rho)|^p |\! \cos(\phi - \gamma)|^p \mathrm{d}\phi$, where $\rho$ and $\gamma$ are the~relevant angles between $\mathbf{x}_i - \mathbf{y}_j$ and the~x-axis. Finally, the~expression
	\begin{align*}
		&\E[
			\wass^p_p((\Pi_{\mathbf{v}_1})_\# \distone,(\Pi_{\mathbf{v}_1})_\# \disttwo)
			\wass^p_p((\Pi_{\mathbf{v}_2})_\# \distone,(\Pi_{\mathbf{v}_2})_\# \disttwo)
		]
		\\
		&=
		\frac{1}{2^2}
		\sum_{k, l \in \{A , B\}}
		\sum_{m, n}^2
			\Prob((\mathbf{v}_1 \mathbf{v}_2) \in E_k \times E_l) \,
			\E[
				|\langle \mathbf{v}, \mathbf{x}_m  - \mathbf{y}_{\sigma_k(m)} \rangle|^p
				|\langle \mathbf{v}, \mathbf{x}_n  - \mathbf{y}_{\sigma_l(n)} \rangle|^p
				\, \vert \, (\mathbf{v}_1 \mathbf{v}_2) \in E_k \times E_l
			]
		\, ,
	\end{align*}
	can be computed in fashion similar to that of the~second moment, with the~integrals now being $\int_{\widetilde{E}_k \cap \{\widetilde{E}_l + \tfrac{\pi}{2}\} \cap [0, 2 \pi)} |\! \cos(\phi - \rho)|^p |\! \cos(\phi - \gamma)|^p \mathrm{d}\phi$.
	
	Putting all these together, an~example of a~dataset for which i.i.d.\ estimation of the~$1$-sliced Wasserstein distance strictly dominates orthogonal is $\mathbf{x}_1 = [1.23, -2.17]^\top$, $\mathbf{x}_2 = [-2, -0.65]^\top$, $\mathbf{y}_1 = [-0.14, -0.93]^\top$, $\mathbf{y}_2 = [-0.82, 0.43]^\top$, where the~MSEs of the~i.i.d.\ and orthogonal estimators for $\projNumber=2$ are respectively $\approx 0.011900$ and $\approx 0.017085$, and the~distance itself equals $\approx 1.082029$. A~dataset for which the~orthogonal coupling dominates is $\mathbf{x}_1 = [1, 1]^\top$, $\mathbf{x}_2 = [0, 0]^\top$, $\mathbf{y}_1 = [0, -\tfrac{1}{2}]^\top$, $\mathbf{y}_2 = [1, -1]^\top$, where the~MSEs of the~i.i.d.\ and orthogonal estimators for $\projNumber=2$ are respectively $\approx 0.073996$ and $\approx 0.006047$, and the~distance itself equals $\approx 0.795774$. Neither i.i.d.\ nor orthogonal estimation thus strictly dominates the~other in terms of MSE across all $p$-sliced Wasserstein distances.
\end{proof}


\section{Appendix on Experiments}

\subsection{Generative Modelling: Auto-encoders}

We consider sliced Wasserstein Auto-encoders (AE) \citep{swae} on MNIST dataset. MNIST dataset contains $50000$ training gray images each with dimension $28 \times 28$. To facilitate HD projection, we augment the images to $32 \times 32$ by padding zeros. Hence in our case the observations have dimension $\mathbf{x}\in \mathbb{R}^{32 \times 32}$. 

\paragraph{Implementation Details.} The AEs have the same architecture as introduced in \cite{swae}. The hidden code $\mathbf{z}$ has dimension $h = 128$. The prior distribution $p(\mathbf{z})$ is chosen to be a uniform distribution inside $[-1,1]^h$. For each iteration, we take a full sweep over the dataset in a random order. All implementations are in Tensorflow \citep{abadi2016tensorflow} and Keras \citep{chollet2015keras}, we also heavily refer to the code of the original authors of \citep{swae} \footnote{\url{https://github.com/skolouri/swae}}. 

\subsection{Generative Modelling: sliced Wasserstein vs. projected Wasserstein}

\paragraph{Background.}
Here we present the comparison between training AE using sliced Wasserstein distance vs. projected Wasserstein distance. Recall that the training objective of the AE is in general
\begin{align}
L_{\text{ae}}(\theta,\phi)\! =\! \mathbb{E}_{\mathbf{x} \sim P(X)}[\|g_\phi(f_\theta(\mathbf{x})) - \mathbf{x}\|]\! + D(p_\theta(\mathbf{z}), p(\mathbf{z})), \nonumber
\end{align}
where $D(\cdot,\cdot)$ can be some proper discrepancy measure between two distributions. In practice, $D(\cdot,\cdot)$ can be KL-divergence \citep{kingma2013auto}, sliced Wasserstein distance \citep{swae} or projected Wasserstein distance (this work). All the aforementioned alternates allow for fast optimization using gradient descent on the discrepancy measures. 

\paragraph{Implementation Details.} The AEs have the same architecture as introduced in \cite{swae}. The hidden code $\mathbf{z}$ has dimension $h = 2$. The prior distribution $p(\mathbf{z})$ is chosen to be a uniform distribution on the interior of the 2D circle with radius $1$. The other implementation details are the same as above.

\paragraph{Results.} We compare the posterior hidden codes $\mathbf{z} \sim p_\theta(\mathbf{z})$ generated by the trained encoders under sliced Wasserstein distance (left) vs. projected Wasserstein distance (right) in Figure \ref{fig:hiddencodes}. Though both hidden code distributions largely match that of the prior distribution $p(\mathbf{z})$, the hidden codes trained by sliced Wasserstein distance tend to collapse to the center (the distribution has a slightly smaller effective support). This observation is compatible with our observations in the generator network experiments presented in the next section.

\begin{figure}\centering
	\null\hfill
	\includegraphics[keepaspectratio, width=.49\columnwidth]{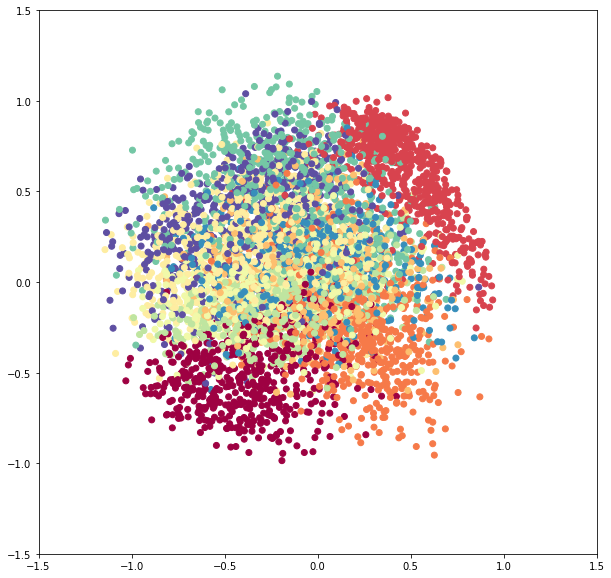}\hfill	\includegraphics[keepaspectratio, width=.49\columnwidth]{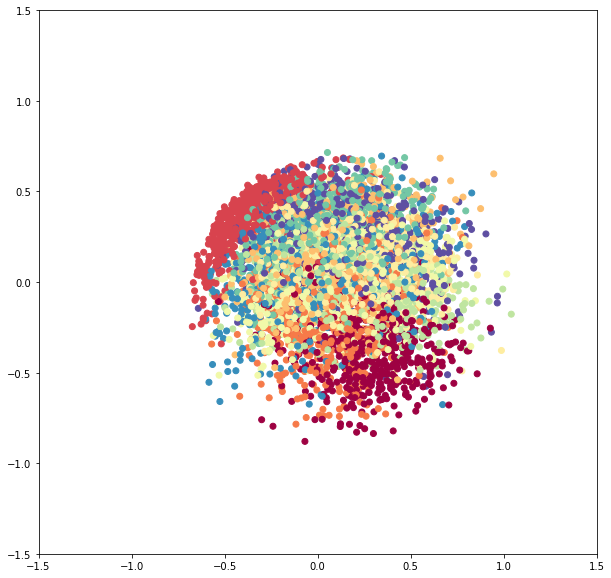}\hfill\null
	\caption{Training AEs using two sets of distance measure (left: sliced Wasserstein distance, right: projected Wasserstein distance): We show hidden codes $p_\theta(\mathbf{z})$ generated by the encoder after training. Though both distributions generally match the prior distribution $p(\mathbf{z})$, the distribution trained with projected Wasserstein distance tends to collapse to the center.}
	\label{fig:hiddencodes}
\end{figure}

\subsection{Generative Modelling: Generator Networks}
\paragraph{Background.}  A generator network $G_\theta$ takes as input a noise sample $\epsilon \sim \rho_0(\cdot)$ from an elementary distribution $\rho_0$ (e.g. Gaussian) and output a sample in the target domain $\mathcal{X}$ (e.g. images), i.e. $X = G_\theta(\epsilon) \in \mathcal{X}$. Let $P_\theta(X)$ be the implicit distribution induced by the network $G_\theta$ and noise source $\rho_0$ over samples $X$. We also have a target distribution $\hat{P}(X)$ (usually an empirical distribution constructed using samples) that we aim to model. The objective of generative modeling is to find parameters $\theta$ such that $P_\theta(X) \approx \hat{P}(X)$ by minimizing certain discrepancies 
\begin{align}
D(P_\theta(X), \hat{P}(X)), \label{eq:generative}
\end{align}
for some discrepancy measure $D(\cdot,\cdot)$. When $D(\cdot,\cdot)$ is taken to be the Jenson-Shannon divergence, we recover the objective of Generative Adversarial Networks (GAN) \citep{goodfellow2014generative}. Recently, \citet{swgenmodels,swgan} propose to take $D(\cdot,\cdot)$ to be the sliced Wasserstein distance, so as to bypass the potential instability due to min-max optimization formulation with GAN. Similarly, $D(\cdot,\cdot)$ can be projected Wasserstein distance. Here, we show the empirical differences of these two generative models (sliced Wasserstein generative network vs. projected Wasserstein generative network) with an illustrating example.

\paragraph{Setup.} We take $\mathcal{X} = \mathbb{R}^2$ and $\hat{P}(X)$ to be the empirical distribution formed by samples drawn from a mixture of Gaussians. The mixture contains $16$ components with centers evenly spaced on the 2-D grid with horizontal/vertical distance between neighboring centers to be $0.3$. Each Gaussian is factorized wit diagonal variance $0.1^2$. The samples are illustrated as the red points in Figure \ref{fig:gan} below. 

\paragraph{Implementation Details.} The generators are parameterized as neural networks which take $2-$dimensional noise (drawn from a standard factorized Gaussian) as input and output samples in $\mathcal{X} = \mathbb{R}^2$. The networks have two hidden layers each with $256$ units, with $\text{relu}$ nonlinear function activation in between. The final output layer has $\text{tanh}$ nonlinear activation. We train all models with Adam Optimizer and learning rate $10^{-4}$ until convergence. All implementations are in Tensorflow \citep{abadi2016tensorflow}, we also heavily refer to a set of wonderful open source projects \footnote{\url{https://github.com/kvfrans/generative-adversial}} \footnote{\url{https://github.com/ishansd/swg}}.

\paragraph{Results.} The results for generative modelling are in Figure \ref{fig:gan} (left for sliced Wasserstein distance, right for project Wasserstein distance). Red samples are those generated from the target distribution. Blue samples are those generated from the generator network after training until convergence. We observe that samples generated from these two models exhibit distinct features: under sliced Wasserstein distance, the samples tend to be more widespread and in this case capture the modes on the perimeter of the Gaussian mixtures. On the other hand, under projected Wasserstein distance, the samples tend to collapse to the center of the target distribution and only capture modes in the middle.

\begin{figure}\centering
	\null\hfill
	\includegraphics[keepaspectratio, width=.49\columnwidth]{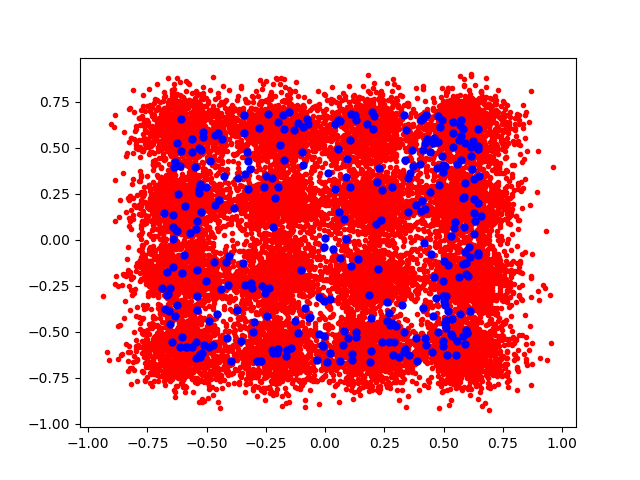}\hfill	\includegraphics[keepaspectratio, width=.49\columnwidth]{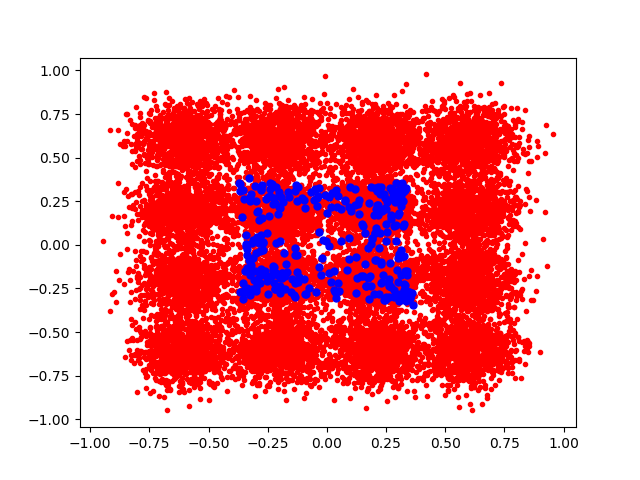}\hfill\null
	\caption{Training generators using two sets of distance measure (left: sliced Wasserstein distance, right: projected Wasserstein distance): Red samples are form the target distributions, which are drawn from a mixture of Gaussians with $16$ mixtures. Blue samples are those generated from the generator network after convergence. Under sliced Wasserstein distance the samples tend to spread out and capture modes at the perimeter of the mixtures, while under projected Wasserstein distance the samples tend to cluster in the center and capture modes in the middle.}
	\label{fig:gan}
\end{figure}

\subsection{Reinforcement Learning}

\paragraph{Background.}  Vanilla policy gradient updates $\theta_{\text{new}} \leftarrow \theta_{\text{old}} + \alpha \nabla_{\theta_{\text{old}}} J(\pi_{\theta_{\text{old}}})$ suffer from occasionally large step sizes, which lead the policy to collect bad samples from which one could never recover \citep{schulman2015trust}. Trust region policy optimization (TRPO) \citep{schulman2015trust} propose to constrain the update using KL divergence $\mathbb{KL}[\pi_{\theta_{old}} || \pi_{\theta_{new}}] \leq \epsilon$ for some $\epsilon > 0$, which can be shown to optimize a lower bound of the original objective and significantly stabilize learning in practice. Recently, \citet{zhang2018policy} interpret policy optimization as discretizing the differential equation of the Wasserstein gradient flows, and propose to construct trust regions using Wasserstein distance. In general, trust region constraints are enforced by $D(\pi_{\theta_{\text{old}}},\pi_{\theta_{\text{new}}}) \leq \epsilon$ for some $\epsilon$, where \citet{schulman2015trust} use KL divergence while \citet{zhang2018policy} use Wasserstein distance.

Instead of constructing the constraints explicitly, one can adopt a penalty formulation of the trust region and apply \citep{schulman2017proximal,zhang2018policy}
\begin{align}
\theta_{\text{new}} \leftarrow \theta_{\text{old}} + \alpha \nabla_{\theta_{\text{old}}} (J(\pi_{\theta_{\text{old}}}) - \lambda D(\pi_{\theta_{\text{old}}},\pi_{\theta_{\text{new}}})), \nonumber 
\end{align}
where $\lambda > 0$ is a trade-off constant. The above updates encourage the new policy $\pi_{\theta_{\text{new}}}$ to achieve higher rewards but also stay close to the old policy for stable updates.

Due to the close connections between various Wasserstein distance measures, we propose to set $D(\cdot,\cdot)$ as either sliced Wasserstein distance or projected Wasserstein distance. Since projected Wasserstein distance corrects for the implicit "bias" in the sliced Wasserstein distance, we expect the corresponding trust region to be more robust and can better stabilize on-policy updates.

\paragraph{Implementation Details.} The policy $\pi_\theta$ is parameterized as feed-forward neural networks with two hidden layers, each with $h=64$ units with $\text{tanh}$ non-linear activations. The value function baseline is a neural network with similar architecture. To implement vanilla policy optimization, we use PPO with very large clipping rate $\epsilon = 10.0$, which is equivalent to no clipping. We set the learning rate to be $\alpha = 3\cdot 10^{-5}$ and the trade-off constant to be $\lambda = 0.001$ for the trust region. All implementations are based on OpenAI baseline \citep{baselines} and benchmark tasks are from OpenAI gym \citep{brockman2016openai}. 

\subsection{Hadamard-Rademacher random matrices}\label{sec:hd}
Here, we give brief details around Hadamard-Rademacher random matrices, which are studied in Section \ref{sec:genmod} as an approximate alternative to using random orthogonal matrices drawn from $\mathrm{UnifOrt}(S^{d-1}; d)$. These random matrices have been used as computationally cheap alternatives to exact sampling from  $\mathrm{UnifOrt}(S^{d-1}; d)$ in a variety of applications recently; see e.g. \citep{HD1,HD2}. A 1-block Hadamard-Rademacher matrix is simulated by taking $\mathbf{H}$ to be a normalised Hadamard matrix in $\mathbb{R}^{d \times d}$, and $\mathbf{D}$ to be a random diagonal matrix, with independent Rademacher ($\mathrm{Unif}(\{\pm 1\})$) random variables along the diagonal. The Hadamard-Rademacher matrix is then given by the product $\mathbf{H}\mathbf{D}$. Multi-block Hadamard-Rademacher random matrices are given by taking the product of several independent 1-block Hadamard-Rademacher random matrices.

\end{document}